\documentclass{sig-alternate-2013-modified}
\usepackage[utf8]{inputenc} 
\usepackage[T1]{fontenc}    
\usepackage{url}            
\usepackage{booktabs}       
\usepackage{amsfonts}       
\usepackage{nicefrac}       
\usepackage{microtype}      
\usepackage{wrapfig}

\usepackage{url}
\usepackage{amssymb,amsmath,amsfonts}
\usepackage{graphicx}
\usepackage{times}
\usepackage{algorithm}
\usepackage[algo2e, ruled, vlined]{algorithm2e}

\newcommand{\ignore}[1]{}
\newcommand{\notinproc}[1]{#1}
\newcommand{\onlyinproc}[1]{}

\newtheorem{thm}{Theorem}[section]
\newtheorem{theorem}{Theorem}[section]

\newtheorem{lemma}[thm]{Lemma}

\newcommand{\ADS}{\mathop{\rm ADS}}

\newcommand{\kth}{\text{k}^{\text{th}}}

\newcommand{\instance}[1]{\textsf{#1}}

\def\E{{\textsf E}}

\def\KLdiv{D_{\text KL}}
\def\Exp{\textsf{Exp}}
\def\kernel{\kappa}

\def\vecf{\boldsymbol{f}}
\def\vecy{\boldsymbol{y}}
\def\vecu{\boldsymbol{u}}

\title{Semi-Supervised Learning on Graphs\\ through Reach 
   and Distance Diffusion 
}
\numberofauthors{1}
\author{\alignauthor Edith Cohen \\
{\large Google Research, USA}\\
{\large edith@cohenwang.com}
}
 \ignore{
\author{
Edith Cohen\\
Google Research, USA\\
\texttt{edith@cohenwang.com}
}
 }
\date{}
\begin{document}
\maketitle
 \begin{abstract}
   Semi-supervised learning (SSL)  is an indispensable tool
when there are few  labeled entities and many unlabeled entities  for which we want to predict labels.
With graph-based methods, entities correspond to nodes in a graph and edges
represent strong relations.   At the heart of SSL algorithms is the
specification of a dense {\em kernel} of pairwise affinity values from the graph structure.
  A learning algorithm is then trained on the kernel together with labeled entities.
The most popular kernels are {\em spectral} and include the highly scalable
 ``symmetric''   Laplacian methods,  that compute a soft labels using Jacobi
 iterations, and ``asymmetric'' methods including Personalized Page
 Rank (PPR) which use short random walks and apply with directed relations, such
as like,  follow, or hyperlinks.

We introduce
{\em Reach diffusion} and {\em Distance diffusion}
kernels that build on powerful social and economic models of centrality and 
influence in networks  and
capture the directed pairwise relations that underline social 
 influence.   Inspired by the success of social influence as an alternative to spectral
 centrality such as Page Rank, we explore SSL with our
 kernels and develop highly scalable algorithms for parameter setting, 
label learning, and sampling. We perform preliminary experiments that 
  demonstrate the properties and potential of our kernels.

  \end{abstract}

\ignore{
that can be more suitable that can qualitatively dominate 

 inspired by tremendously successful 

  Social diffusion, studied in the seminal work of [Kempe,
  Kleinberg, and Tardos 2003] and rooted in classic economic and
  social network models had 
  We formalize here a novel framework for semi-supervised learning
  that is inspired by social diffusion.  We define  {\em Reach Diffusion} and
  {\em Distance Diffusion}

 Two powerful techniques to define such a kernel are ``symmetric''  spectral
 methods and Personalized
  Page Rank (PPR).  With spectral methods, 
labels can be scalably learned using Jacobi iterations, but
an inherent limiting issue is that they 
are applicable to 
{\em  symmetric} (undirected) graphs, whereas
often, such as with like, follow, or hyperlinks, relations
 between entities are inherently asymmetric.  
PPR naturally works with directed graphs but even with state of the
art techniques does not scale when we want to learn billions of labels.

Aiming at both high scalability and handling of directed relations,  
we propose here {\em Reach Diffusion} and {\em Distance Diffusion} kernels.
Our design is 
inspired by models for influence diffusion in social networks,
formalized and spawned from the seminal work of
[Kempe, Kleinberg, and Tardos 2003].  
 We tailor these models to define a natural 
asymmetric ``kernel'' and design
}

\section{Introduction}

  Semi-supervised learning (SSL)
  \cite{BlumChawla:ICML2001,ZhuGL:ICML2003,SemiSupervisedLearning:2006}
  is a fundamental tool in  applications when there are
few labeled (seed) examples $(x_j,\vecy_j)$ $j\leq n_\ell$ and many $n_u \gg n_\ell$ unlabeled examples $x_i$ for $i\in (n_\ell,n_\ell+n_u]$.
SSL algorithms utilize some 
auxiliary structure, for example a metric embedding or interaction graph
 on the space of examples, from which a kernel $\kernel$ of pairwise {\em affinities}  is derived.
The goal is to predict labels for the 
unlabeled examples 
that are as  consistent as possible with seed labels $\vecy_j$ and 
affinities --  so that a learned label of an example is more similar to seed examples 
that are more strongly related to it.


A suitable kernel is critical to the
quality of the learned labels, but we first briefly discuss the almost
orthogonal issue of how a kernel is used.
A common choice is the nearest-neighbors classifier which in our context uses the
soft label 
\begin{equation}\label{kNN:eq}
\vecf_i=\frac{1}{k}\sum_{j\in \text{\rm
    top-k}\{\kernel_{ij}\mid j\leq n_\ell\}} \vecy_i
\end{equation}
 which is
the average of the labels vectors of the $k$ seeds that are nearest to
$j$ (have highest $\kernel_{ij}$) \cite{CoverHartkNN:InfoTheory1967}.
A more refined goal is to
minimize the squared loss
\begin{equation} \label{basicopt}
\sum_{i > n_\ell} \sum_{j \leq n_\ell} \kernel_{ij} ||\vecf_i-\vecy_j||_2^2\ . 
\end{equation}
The solution
is a  weighted average of the seed labels:
\begin{equation} \label{watsonnadaraya:eq}
\vecf_i = \frac{\sum_{j \leq n_\ell} \kernel_{ij} \vecy_j}{\sum_{j\leq n_\ell} \kernel_{ij}}\ . 
\end{equation}
This expression is known as the {\em kernel density estimate} (Watson
Nadaraya estimator \cite{Watson:1964,Nadaraya:1964} which builds on 
\cite{Rosenblatt:stats1956,Parzen:MathStats1962,Silveman:monograph1986}).
Nearest-seeds and kernel density are most often used when $\kernel$ is 
positive semi definite, but here we view it more generally.
In particular,  the expression \eqref{watsonnadaraya:eq} is a solution of the optimization \eqref{basicopt}
also when $\kernel$ is asymmetric \cite{Silveman:monograph1986}.
\ignore{
A special case of this estimator is 
the celebrated $k$-nearest neighbors (kNN) classifier 
\cite{CoverHartkNN:InfoTheory1967}:
Assuming some metric embedding of points, the affinity relation is
$\kernel_{ij}$ if $j$ is one of the $k$ closest labeled examples to $i$, and
$\kernel_{ij} = 0$ otherwise.
}
 The prediction of labels from soft labels can be direct, for example,
 using the maximum entry, or by
training a supervised learning algorithm on soft labels and labels of
seeds.  A very different use of the kernel,
inspired by 
the huge success of 
word embeddings \cite{Mikolov:NIPS13}, is to embed 
entities in latent
feature space of small dimension so that relation between embedding
vectors approximates respective kernel entries.
A supervised learning algorithm is
then trained on embedding vectors and labels of seed nodes
\cite{deepwalk:KDD2014,node2vec:kdd2016,Yang:ICML2016}.

 We now return to the specification of the kernel.
Seeds are typically a small fraction of examples and therefore it is 
critical that our kernel meaningfully captures weak affinities.
The raw data, however, typically only contains 
strong relations $w_{ij}$ in the form of 
pairwise interactions between entities:  Friendships in a social network, word co-occurrence 
relations, product purchases,
movie views by users, or 
features in images or documents. 
The interactions 
strengths may reflect frequency, recency, confidence, or 
importance or learned from node features.
The strong affinity values are represented by a graph with entities as nodes and
interactions as weighted edges.
The raw data is often enhanced by 
embedding entities in a lower dimensional Euclidean space 
so that 
larger  inner products, or closer distances, between the embedding vectors fit the provided 
 interactions \cite{Koren:IEEE2009}.   Such
embeddings of strong interactions are hugely successful in identifying 
strong interactions that were not explicit in the raw
 data (aka link prediction \cite{Liben-Nowell_Kleinberg:CIKM2003}), but the dense 
kernel they define is typically not accurate for weak relations.   The
embeddings are used to construct a more precise sparse 
graph of strong relations
\cite{TSL:science2000,hastietibshirani:book2001,ZhuGL:ICML2003}.
\notinproc{
To visualize an embedding that fits only the strong relations
consider points 
that form a dense manifold that lies in 
a higher  dimensional ambient space 
\cite{LLE:science2000,TSL:science2000,ZhuGL:ICML2003,Orlitsky:ICML2005,BRSrebro:UAI2011}, where
weak relations correspond to distances over the manifold.
}



The all-range kernel $\kernel$ we seek extends the
provided strong affinities by considering
the {\em ensemble} of paths from $i$ to $j$. We expect it to 
satisfy some principles of network science:  Increase
with the strength of edges, for shorter paths, and when there are more
independent paths between entities.
In addition, it is often helpful to discount connections
through high degree nodes, and be able to tune, via hyper parameters,
the effect of each property.

The most popular SSL kernels are {\em spectral} \cite{BlumChawla:ICML2001,ZhuGL:ICML2003,SemiSupervisedLearning:2006}.
``Symmetric''  methods compute learned labels that solve
an optimization problem with
smoothness terms of the form $w_{ij} ||\vecf_i -\vecf_j||^2$
which encourage learned labels of points with high 
$w_{ij}$ to be more similar and terms of the form $\lambda ||\vecf_i -\vecy_i ||^2$ for $i\leq n_\ell$
that encourages learned labels of seed nodes to be close to the true labels.
One such objective was proposed in the
influential {\em label propagation} work \ignore{by Zhu et al} 
\cite{ZhuGL:ICML2003}.
Related objectives,  {\em adsorption} and {\em modified adsorption}
were studied for YouTube
recommendations and named-entity recognition
\cite{Adsorption:WWW2008,TalukdarCrammer:ECML2009,TalukdarPereira:ACL2010}.    
The solution can be expressed as a set of linear equations of a
particular diagonally dominant form and computed by
inverting a corresponding matrix.  We can view this inverted matrix as
an all-range kernel $\kernel$ (which does 
not depend on the labels of the seed nodes), and the learned labels
are density estimates with respect to $\kernel$. 
\notinproc{
Other interpretations of the solution are as a fixed point of a
stochastic sharing process or the landing probability of a random walk
\cite{Chung:Book97a,KonLaf:ICML2002}.}
In practice, the dense $\kernel$ is not explicitly computed or stored
and instead the solution specific to the seed labels is approximated using the Jacobi method.
 The computation of each  gradient update is linear in the
number of edges and often tens or hundreds of iterations suffice\notinproc{ \cite{FujiwaraIrie:ICML2014}}. 
\ignore{
A further optimization sparsifies  the set of 
unlabeled points using a smaller set of {\em
  anchors} that is large enough to preserve the short-range structure
but is much smaller than the full set 
\cite{DelalleauBLeR:AIstat2005,LiuHC:ICML2010}.  Anchors are selected
as samples or cluster representatives.
Other unlabeled points are expressed as weighted combinations of 
anchors, removed from the ``spectral'' computation, and eventually
inherit as their learned label, an appropriate linear combination, in
essence a density estimate,  of the learned labels of the anchors.
}

Many interactions, 
such as 
follows, hyperlinks, and likes, are
 inherently asymmetric.  For them, symmetric spectral methods, which require 
undirected relations to guarantee convergence, are less suitable but
other spectral methods can be used.
  Laplacians for directed graphs 
\cite{dirLaplacian:FanChung:2006,PerraultMeila:NIPS2011,Yoshida:wsdm2016}
often result in algorithms that are not as elegant or as scalable.
A particularly successful technique is
Personalized Page Rank (PPR) and variants that use short random walks
\cite{Chung:Book97a,pagerank:1999}.  
PPR can be either personalized to 
labels (all seed with the same label) or to each individual
unlabeled nodes $i$.  Note that the two approaches are semantically
and computationally very different because random walks are inherently not reversible.
In the former, the kernel relates labels to nodes
\cite{ZhouBLWS:NIPS2004,LinCohen:ICWSM2008}, essentially ranking nodes from the
persepctive of each label, and can be approximated by
set of Jacobi iterations using the label dimension, which scales
well when the number of labels is small.
In the latter, the entry $\kernel_{ij}$
corresponds to the probability of visiting $j$ when personalizing to
$i$,  ranking (groups of seed nodes with the same) labels from the perspective
of each node.  While arguably this is what we want, computation via
Jacobi iterations involves propagating vectors with the node
dimension, and computation by simulating random walks from each $i$ requires that
they reach sufficiently many
 seeds, which scales poorly with sparse seeds 
even with state of the art techniques \cite{Lofgren:KDD2014}.  This is
because fast PPR designs aim to identify nodes
with largest visiting probabilities, whereas we require
{\em seed} nodes with largest visiting probabilities.  When
the seed set is a small fraction of all nodes, its total 
visiting probability is small.  This 
means that any algorithm from basic Monte Carlo generation of walks to
the bidirectional approach of \cite{Lofgren:KDD2014} would
spend most of its ``work''  traversing non-seed nodes.

 Finally, as mentioned, PPR and related kernels are used in
\cite{deepwalk:KDD2014,node2vec:kdd2016,Yang:ICML2016} to compute an
embedding.  In this case,  the only access required to the kernel is
for generating weighted samples of related pairs of entries, and is achieved by
simulating short walks.  The embedding alleviates the need 
to perform enough walks to reach seed nodes but necessitates the
computation of an embedding.

\onlyinproc{\paragraph {Contributions}}\notinproc{\subsection {Contributions}}
Inspired by the success of distance-based and reachability-based
centrality and influence as a formidable alternative to 
spectral notions in social and information networks, we faciliate their
application to SSL.
We define {\em reach diffusion} and {\em distance
  diffusion} kernels that capture the directed pairwise relations
underlying these classic influence definitions.
In a nutshell, our two models, reach and distance, complement each other in that they
capture two  different prevelant  interpretations of edges in datasets 
 represented as graphs.  With reach, the quality of a path depends 
 on its weakest link.  With distance, it depends on the sum of its
 links.  With both, the quality of a cut depends on
 its strongest link.  The use of randomization allows us to factor in
 the redundancy in the cut or more generally, the connecting path
 ensembles.  Both models offer distinct qualities than spectral
 models, which capture hitting probabilities of various random walks.

We then facilitate scalable application of reach/distance diffusion  kernels to
SSL through state of the art sketching techniques.
The computation of a sketch for each node is near-linear in the
size of the input.  From the sketch of each unlabeled node we obtain a
respective approximate kernel-density soft label.
  We establish statistical guarantees on the estimate quality of the
approximate labels with respect to the exact ones as defined in the
model. Moreover, the sketches also provide weighted samples of the
kernels which can be used to compute an embedding, essentially
replacing the spectral kernel component in
\cite{deepwalk:KDD2014,node2vec:kdd2016,Yang:ICML2016}  which is
sampled using random walks.
We perform a preliminary experimental study that demonstrates the application
and potential of our kernels. \onlyinproc{Due to page
  limitation, many details are
  omitted and we refer the reader to the included full version.}

\ignore{
\begin{wrapfigure}{r}{0.25\textwidth}
\includegraphics[width=0.25\textwidth]{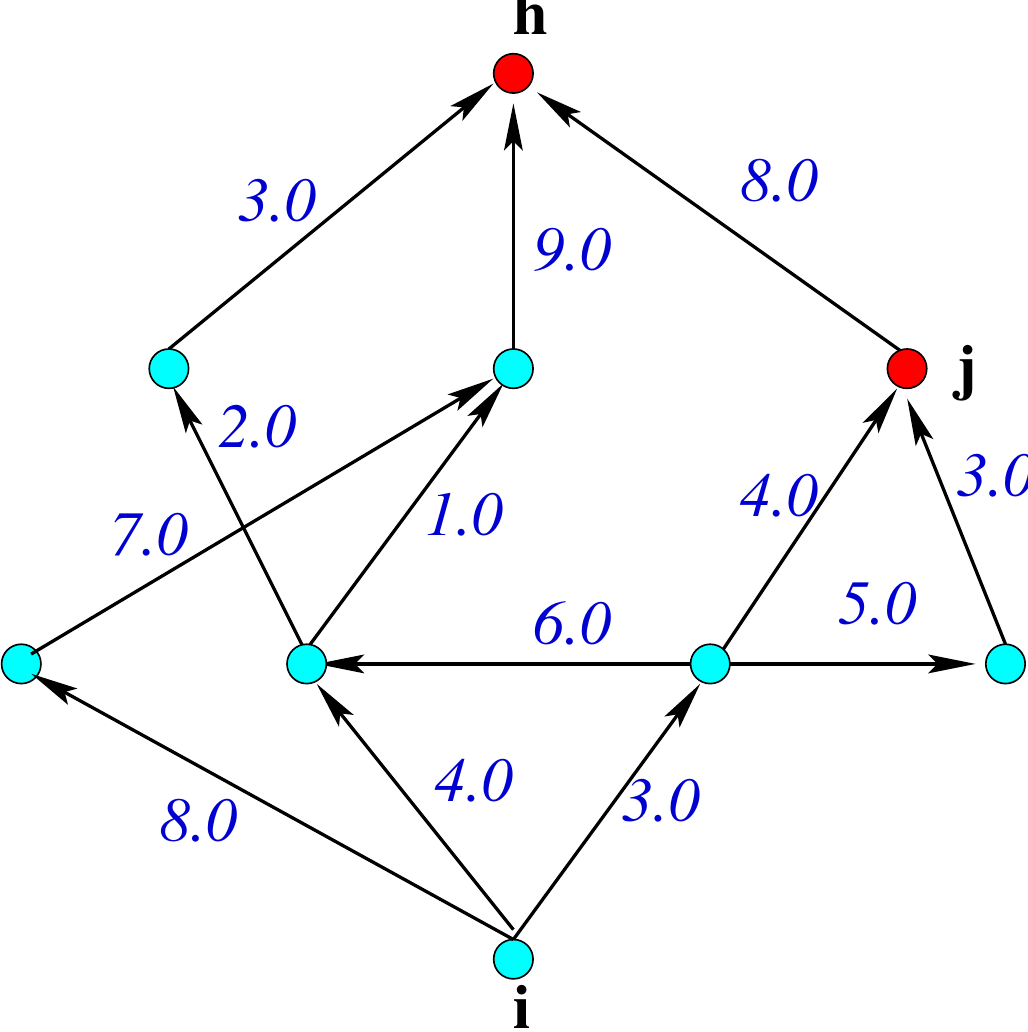}
\onlyinproc{\vspace{-15pt}}
\caption{{\small Graph with edge lengths/lifetimes}}
\label{example:fig}
\end{wrapfigure}
}

\begin{figure}
\center
\includegraphics[width=0.25\textwidth]{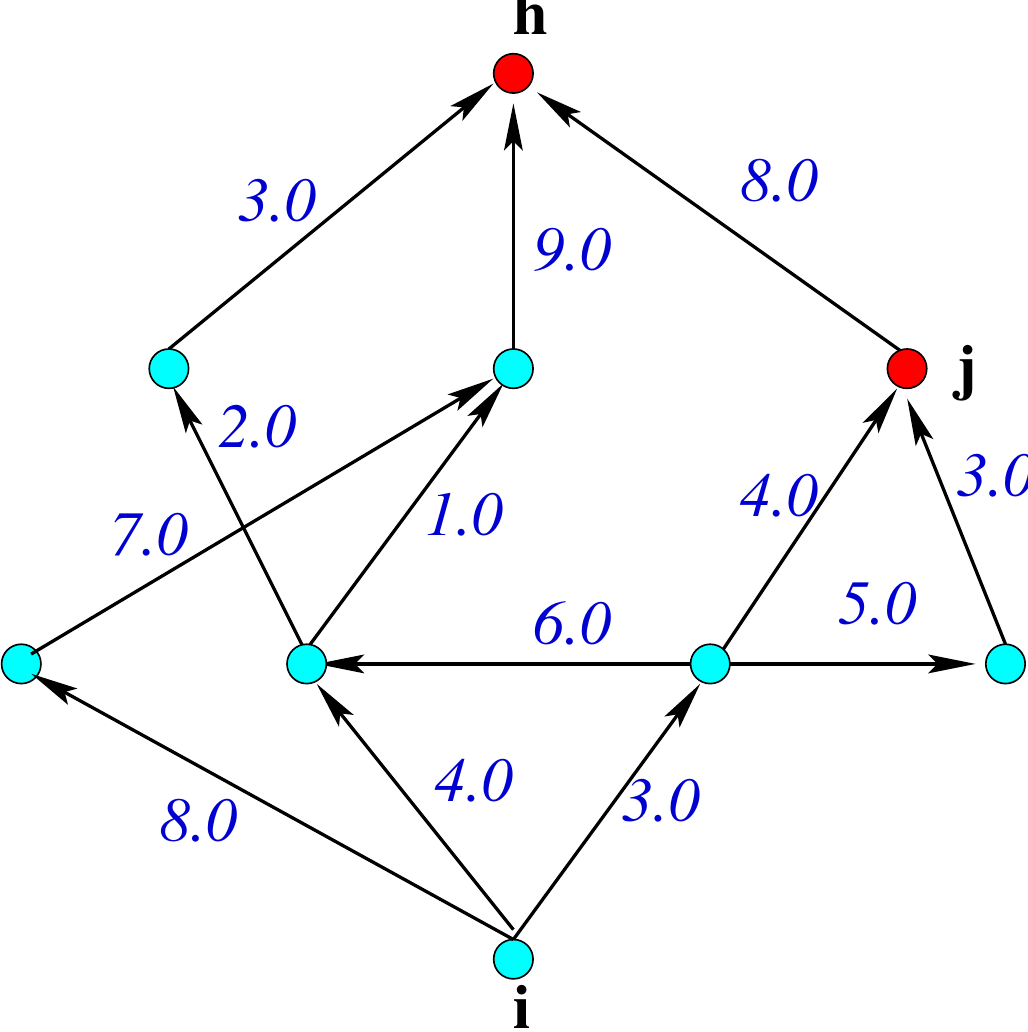}
\caption{{\small Graph with edge lengths/lifetimes}}
\label{example:fig}
\end{figure}

\onlyinproc{\section{Reach Diffusion Kernels}}
\notinproc{\subsubsection{Reach diffusion}}
Our reach diffusion kernels are inspired by popular
information diffusion models 
motivated
by Richardson and Domingos
\cite{RichardsonDomingos:KDD2001} 
and formalized by Kempe, Kleinberg, and Tardos \cite{KKT:KDD2003}
and by the field of
reliability or survival analysis
\cite{MillerSurvivalAnalysis:book,lawless2011statistical} applied to
engineered and biological systems.

Influence diffusion \cite{KKT:KDD2003} is 
defined for a network of
directed pairwise interactions between entities.
 A probability distribution on the subset of active edges is
constructed, and the {\em influence} of a 
node $v$ is then measured as  the expected number of nodes $v$ 
can reach through active edges.  
Independent Cascade (IC) \cite{KKT:KDD2003}, which uses
independent activation probabilities $p_e$ to edges, is 
the simplest and most studied model.
The influence of a node, when defined this way, satisfies the
desirable properties of increasing when paths to other nodes are
shorter and when there are more independent paths.

To apply this approach to SSL, we need to first define
an appropriate kernel $\kernel_{ij}$ that provides corresponding pairwise
``influence'' values.  
The straightforward first attempt is to define $\kernel_{ij}$ as the probability
that $i$ reaches $j$.  But this has scalability issues, similar to
PPR, when the seed set is sparse:  Approximation of the kernel density estimates require
that we compute these probabilities for ``sufficiently many'' seeds.
Instead, we propose a refinement 
that both scales and satisfies desirable
properties. 

Inspired by reliability analysis, we view edges as a
 component of a system connecting entities.
 We associate with edges continuous random variables $\mu_e$ that
 correspond to their {\em lifetime} and chosen so that the expected lifetime increases
 with the ``significance'' of the edge.
Edges that correspond to more significant interactions have higher expected lifetimes.
From this, we can define for each ordered pair of nodes $(i,j)$ its
{\em survival time threshold}  random variable $t_{ij}$, which is
the maximum $\tau$ such that $j$ is reachable from $i$ via edges with 
lifetime $\mu_e \geq \tau$. 
If $j$ is connected to $i$ via a single directed path,  the survival 
time $t_{ij}$ is 
the minimum lifetime of a path edge.  For a particular node, the 
survivability of having an out connection is the maximum lifetime of an out edge.

 Note that we can express the IC model of \cite{KKT:KDD2003} in terms of
this reliability formulation by choosing independent lifetime variables $\mu_e \sim \Exp[1/p_e]$
(exponentially distributed with parameter $1/p_e$).  The influence of a node $i$ in the IC
model is then the expected number of nodes reachable from $i$ via
edges with $\mu_e \geq 1$.

In a Monte Carlo simulation of the model we obtain a set of lifetime
values $\mu_e$ for edges which imply corresponding survival times
$t_{ij}$ for the connectivity from $i$ to $j$.
We then define $\kernel_{ij}$ as non-decreasing function of
$t_{ij}$ or alternatively, use $\kernel_{ij} = \alpha(N_{ij})$, where
$\alpha$ is non-increasing and 
$N_{ij}$ is the position of node $j$ in a decreasing order of $t_{ij}$.
When $j$ is not reachable from $i$ we define $\kernel_{ij}=0$.
In the example of Figure~\ref{example:fig}, we have $t_{ih}=7$ and $N_{ih}=4$, since there are in total $4$ nodes $a$ with $t_{ia}\geq 7$.
We have $t_{ij}=3$ and $N_{ij}=8$, since all nodes $a$ except one have $t_{ia}\geq 3$.

Note that our kernel $\kappa$ here is also a random variable, which 
assumes values in each simulation of the model.
Our learned labels will be the expectation of the density estimates
over the distribution of $\kernel$.  Our use of randomized kernels
resembles other contexts \cite{RahimiRecht:NIPS2007}.

We can verify that our kernel satisfies the qualitative properties
we seek: Higher significance 
edges, shorter paths, and more independent paths lead to 
higher expected survival times $t_{ij}$ and lower $N_{ij}$.

The position $N_{ij}$ does not only depend on the connectivity ensembles but also on
how the ensemble relates to the corresponding ensembles of other
nodes.   For example, suppose $i$ connects to $j$ via a path of length
$3$ and to $h$ via a path of length $2$ with independent iid $\mu_e$.
Then we always have $\E[t_{ij}] \leq \E[t_{ih}]$ and $\E[N_{ih}] \leq
\E[N_{ij}]$ (the shorter path has better survival).
When the paths are
independent,  however, it is possible to have simulations with
$t_{ih}< t_{ij}$ and $N_{ih} > N_{ij}$
but when the 2-path is the prefix of the 3-path we always have
$t_{ij}<t_{ih}$.

A typical choice for lifetime variables in reliability models is
the Weibull distribution.  If we use Weibull distributed $\mu_e$, with
shape parameter $\beta$ and scale parameter $\lambda$ equal to the
significance of $e$, we obtain some
compelling properties.
 Note that the Weibull family includes the
exponential distribution which is Weibull with shape parameter
$\beta=1$ and corresponds to ``memoryless'' remaining lifetime.
Parameters $\beta < 1$ model higher probability of ``early failures''
and $\beta > 1$ to bell shaped lifetimes concentrated around the expectation.
For two edges with iid lifetimes, 
the probability of one having a higher lifetime than the other is 
proportional to the ratio of their significances to the power of
$\beta$.
From the {\em closure under minimum} property of the Weibull distribution, 
the survival time of a directed path with independent Weibull lifetimes
 is also Weibull distributed with the same shape parameter $\beta$ and scale parameter equal to an inverse of the $\beta$-norm of the vector of 
inverted edge significances.  For exponential distributions, the
expected lifetime of each edge is the inverse of its significance, and
the survival threshold of the path has parameter
 (which is the inverse of the Weibull parameter) equal to the sum
of significances, which yields expected survival that is the inverse
of that sum.  The shape parameter $\beta$ allows us to tune the
emphasis of lower significance edges on the survival of the path.

\onlyinproc{\section{Distance Diffusion Kernels}}\notinproc{\subsubsection{Distance diffusion}}

 Our distance diffusion kernels are inspired by a generalization,
 first proposed by Gomez-Rodriguez et al 
\cite{Gomez-RodriguezBS:ICML2011,DSGZ:nips2013,timedinfluence:2015}
of  the influence model of Kempe et al \cite{KKT:KDD2003}  to a
 distance-based setting. They are also inspired by models of distance-based
 utility in networks
 \cite{CoKa:jcss07,BlochJackson:2007,JacksonNetworks:Book2010} where
 the relevance of a node to another node decreases with the distance
 between them.    In these influence models,  edges have {\em length} random
 variables, which can be interpreted as propagation times. The influence of a
 node $v$ is then defined as a function of elapsed
 ``time'' $T$,  as the expected number of nodes that are activated within a
time $T$ (the shortest-path distance from $v$ is at most $T$).  
Note that the ``time'' here refers to propagation and activation times
rather than ``survival'' time, so shorter times correspond to stronger
 connections.  
To prevent confusion, we will use the terms edge {\em lengths} in the 
context of distance diffusion here and use {\em time} only in the context of 
reach diffusion.

More precisely, we associate length random variables $\ell_e$ with edges
with expectation that {\em decreases} with the edge significance.
 In a simulation of the model we obtain a set of lengths
 $\ell_e \geq 0$ for edges which induces a set of shortest paths
 distances $d_{ij}$.
Again, the random variable $d_{ij}$ depends on 
the ensemble of directed graphs from $i$ to $j$.
A choice of Weibull distributed lengths with scale parameter
equal to the {\em inverse} significance seems particularly 
natural \cite{ACKP:KDD2013,DSGZ:nips2013,CDFGGW:COSN2013}:
The closest out connection from a node corresponds to
the minimum length of an out edge.  When edge lengths are Weibull, the
minimum is also Weibull distributed with the same shape parameter and
a scale parameter equal to the inverted $\beta$-norm of the edge significances.  

Our kernel $\kappa_{ij}$ can be naturally defined through a decreasing function
$\alpha$ of the shortest-path distance $d_{ij}$ or of the position
$N_{ij}$ of node $j$ by increasing distance from $i$. When $j$ is not reachable from $i$ we 
define $\kappa_{ij}=0$.
In the example of Figure~\ref{example:fig}, we have $d_{ih}=9$ and $N_{ih}=9$, since for all nodes $a$ with $d_{ia}\leq 9$.
We have $d_{ij}=7$ and $N_{ij}=6$, since there are 6 nodes $a$ with $d_{ia}\leq 7$.

\onlyinproc{\section{Kernel sketching}}\notinproc{\subsubsection{Kernel
    sketching}}
In the probabilistic model, 
our nearest-seeds or kernel-density soft labels $\vecf_i$ are an
expectation, over Monte Carlo simulations of the model, of a deterministic soft label obtained
in each simulation.
\ignore{
\begin{equation} \label{llabel}
\vecf_i = \E\left[\frac{\sum_{j\leq n_\ell} \kernel_{ij} m_j
    \vecy_j}{\sum_{j\leq n_\ell} m_j \kernel_{ij}}\right]\ ,
\end{equation}
where $m_j >0$ are importance weights that can be attached to entities. 
}
We approximate this expectation by an average and
establish \onlyinproc{using Hoeffding's inequality} that a small number of simulations suffices
to estimate the entries of $\vecf_i$ within small additive error.

The main algorithmic challenge is obtaining a scalable approximation 
of each simulation.  We first consider the simple ``closest seed(s)'' (nn) kernel
weights, where only the $k$ closest seed nodes to
$i$ contribute to $\vecf_i$.  In this case
we need to compute for each $j$ the seed node(s) $i\in U$ with
minimum $d_{ij}$ with distance diffusion or with highest $t_{ij}$ with
reach diffusion.  
For distance diffusion and $k=1$, the computation is
equivalent to a single application of Dijkstra's algorithm (with
appropriate heap initialization with all seeds).  With $k$
  closest seeds the computation is equivalent to $k$ Dijkstra's. For reach diffusion, we develop and
analyse a
{\em survival threshold} graph search which is computationally similar
to Dijkstra: In a nutshell, the summation operation used for shortest paths length 
can be replaced (carefully) with a min operation for tracking 
survival thresholds instead of distances. 

The exact computation of density estimates, however, is prohibitive
when $\kernel$ is dense:
The computation of
$\kernel_{ij}$ and $N_{ij}$ for all seed nodes $j$ and 
unlabeled nodes $i$ uses $n_u$ graph searches, which is
$O(|E| n_u)$ operations and quadratic even for sparse graphs.
  We use instead sketches of $\kernel$ which 
are both computed very effcieintly and allow us to approximate the entries of 
$\vecf_i$ to within small additive errors.
We apply a sketching technique of
reachability sets and of neighborhoods of directed graphs \cite{ECohen6f,ECohenADS:TKDE2015}.
We will use these sketches, computed with respect to different base
sets of nodes, for two different purposes.  The first is to obtain estimates with
small relative error on $N_{ij}$ from the survival threshold $t_{ij}$ with
reach diffusion and from the shortest-path distance  $d_{ij}$ with
distance diffusion.   These estimates replace the expensive exact
computation of kernel entries $\kernel_{ij}$.  The second is to obtain, for each node $i$,
a small tailored weighted sample of seed nodes according to 
the kernel entries $\kernel_{ij}$.  
 Since the sample is appropriately weighted, we can use only the sampled entries with inverse probability weights to approximate the density estimates and
yet obtain a good approximation of the full sums.

 The distance-sketching technique can be applied almost out of the box
 for distance diffusion.  For reach diffusion, however, we need to
 sketch survival times and not distances.  For the first part,
 we need to obtain sketches that will allow us to estimate the set
 sizes $R_{\tau}(i)$ for all $i$ and $\tau$.  For the second part, we
 need to obtain a weighted sample with respect to the reach diffusion kernel.
To do so, we design a threshold 
sketching algorithm which builds on the basic distance-sketching
design \cite{ECohen6f,ECohenADS:TKDE2015} but replaces the 
shortest-path searches by  our ``survival threshold'' graph searches.
We show that total computation of these threshold searches is near-linear and establish its correctness.


 An advantage of our framework is that we can use the same sets of sketches
to compute soft labels $\vecf_i$ with respect to
 multiple kernel weighting options.  Moreover, we also obtain
 leave-one-out soft labels $\vecf'_i$ for seed nodes $i\in U$ which
 depend only on labels of other seed nodes $U\setminus\{i\}$.
This is useful for selecting the kernel weighting that is most
effective for the seed labels ``training set,'' for example, one that minimizes
$\sum_{i\in U} ||\vecf'_i - \vecy_i ||_2$.
Moreover, the set $(\vecf'_i,\vecy_i)$ can be used to address a
separate problem, which is learning the class $\vecy_j$ from the
soft label $\vecf_j$, by using these pairs as labeled examples to
train a model.  We demonstrate such usage in our experiments.

\ignore{
\subsubsection{Parameter setting}
The last component of our framework is a methodology for parameter setting.
There are multiple hyper parameters in our models including the 
significance weights and the lifetime or length random
variables we associate with the graph components.
Our sketches support leave-one-out cross validation on the seed nodes
with the same cost of computing the sketches.  Therefore the seed set
itself can be used to set the hyper parameters.
}

\notinproc{
\subsection{Overview}
In Section~\ref{model:sec} we present
our reach and distance diffusion kernels.
In Section~\ref{sketch:sec} we show how we use Monte Carlo simulations and
sketches to approximate kernel-density soft labels.  We also analyze the
worst-case statistical guarantees on approximation quality that we can obtain.
In Section~\ref{algo:sec} we present algorithms to compute the
approximate labels.  Parameter settings and the derivation of ``hard''
labels from learned soft labels is discussed in
Section~\ref{hyperpar:sec}.  Section~\ref{exper:sec} contains preliminary
experiments that also demonstrate how the models are applied.
}

\notinproc{ 
\section{Model} \label{model:sec}

 Our input is specified as a graph $G=(V,E)$, where the nodes $V$ are 
{\em entities} and edges $E$ (undirected or 
  directed) correspond to interactions between entities.  We associate 
  {\em weights} $w_{e}$ with edges $e\in E$
that reflect the strength of the interaction and
inverse cost of connecting through the head entity.
  We can also associate weights $w_v$ with a node $v$  that reflect the inverse
  cost of connecting through the entity.
The weights, in general, can be learned from node and edge features.
Simple and effective weights 
regularize by degree (number of
interactions) to discount 
connections through
higher degree nodes and to increase
 edge weights by frequency or recency of the interaction.

\subsection{Reach diffusion kernels}
 We build a probabilistic model from this input by associating {\em
lifetime} random 
 variables with edges and nodes.
A natural choice is to  use  for each component $x$,
a Weibull or an exponentially distributed random 
variables $\mu_x \sim \Exp[1/w_x]$ with parameter equal to its weight
$w_x$.  Some components that are ``fixed'' have $t_x = +\infty$.
Note that the expected lifetime is
$\E[\mu_x] = w_x$, so stronger
interactions have longer lifetimes.
In each Monte Carlo simulation of the model we obtain a set of
lifetimes $\mu_x$ for the components of the graph (edges and nodes).  

For a threshold parameter $\tau$,  the set of {\em active} components
are the edges and nodes  $\{x \in E \cup V \mid \mu_x \geq \tau\}$.
For a pair of nodes $(i,j)$, we define the {\em survival time} $t_{ij}$ of the connection from $i$ to $j$ as
the maximum $\tau$ such that  $j$ is reachable from $i$ using 
components with $\mu_x \geq \tau$.  
Note that for an edge $e=(i,j)$ we always have
$t_{ij} \geq \mu_{e}$.   

 Reach diffusion kernels can use absolute survival times, $\kernel_{ij}
= t_{ij}$, or instead use their ranks.
To do so, we use the notation 
$$R_\tau(i) = \{j \mid t_{ij} \geq \tau\}$$ 
for the set of nodes reachable from $i$ via active components. 
Note that the $R_{\tau}(i)$ is a random variables.
Note that for a fixed simulation,  the set of active components and
the reachability sets
$R_{\tau}(i)$ are non increasing with $\tau$.

 For each $i$, the survival times $t_{ij}$ induce an order 
over nodes $j$ where nodes with better ``connectivity'' to $i$ are (in expectation) earlier in this order.
The position of $j$ in this order is captured by the random variable 
$$N_{ij} = |R_{t_{ij}}(i)| = |\{h \mid t_{ih} \geq t_{ij} \}|\ .$$
Finally, we define a (rank based)  reach diffusion kernel 
$$\kernel_{ij} = \alpha(N_{ij})\ ,$$ where $\alpha \geq 0$ is 
non-increasing.  A natural default choice 
is $\alpha(x)=1/x$, where the affinity of $j$ to $i$ is  inversely 
proportional to the number of nodes that precede it in the influence order. 
Another natural choice, is to use a very fast growing $\alpha$, which gives as
a nearest neighbor classifier.
When $j$ is not reachable 
from $i$, we define $\kernel_{ij}=0$.

In the simplest scheme, the lifetimes $t_x$ of different
components can be independent.   
Semantically, this
achieves the effect of rewarding 
multiple edge-disjoint paths, even when they traverse the same nodes. 
(Nodes are not considered failure points).  In general, however, we can also capture correlations between edges by correlating accordingly the lifetime random variables.  For example, we 
can consider all edges with the same head entity as related 
and share the same lifetime $t$, or correlated lifetimes.

 A natural extension is to associate
{\em mass} $m_i \geq 0$ with nodes, that is interpreted as proportional to 
the importance of the example.  For the case when entities correspond to consumers and goods and we are only interested in labeling goods, this flexibility allows us to assign positive mass only to ``goods'' nodes and 
$m_i=0$ to ``consumer'' nodes. 
The relevance of an example $j$ to another $i$ 
is then proportional to its mass, but
inversely depends on the mass that is reached before $j$.
To model this, we refine  the definition to be
$$N_{ij} = \sum_{h \mid t_{ih} \geq t _{ij}} m_h$$ 
as the mass that is reached at the survival threshold of the connection $(i,j)$.
 Our derivations and algorithms can be 
adapted to incorporate mass but for simplicity of presentation, we
focus on the basic setting where $m_i\in\{0,1\}$.

\subsection{Distance diffusion}

 We associate {\em length} nonnegative random variables with edges and nodes.  In
 each Monte Carlo simulation of the model we obtain a fixed set of
 lengths $\ell_x$ for the components of the graph.
We can now consider shortest-paths distances $d_{ij}$ with respect to
the lengths $\ell$.  
The length of a path is defined as the sum of the
lengths of path edges and the lengths of middle nodes of the path.
The distance $d_{ij}$ is the length of the shortest path.
We can again define the kernel according to aboslute distances as
$\kernel_{ij}=\alpha(d_{ij})$ or their ranks.
For convenience here, we overload the notation we used for reach
diffusion:
For $\tau\geq 0$ and node $i$, we denote by $R_\tau(i) = \{j \mid
d_{ij}\leq \tau \}$ the set of nodes $j$ within distance at most
$\tau$ from $i$.  For nodes $i,j$, we denote by $N_{ij}$ the number
(or mass) of nodes $h$ with $d_{ih} \leq d_{ij}$.
The (rank based) distance diffusion kernel is defined as $\kappa_{ij} = \alpha(N_{ij})$.

\subsection{Kernel distribution and prediction}
In the semi-supervised learning setup, a subset
of the nodes, those with $j\leq n_\ell$ have provided labels $\vecy_j$
and we use the kernel to compute soft labels for nodes $i > n_\ell$.  With
nearest-seeds \eqref{kNN:eq} or kernel density
\eqref{watsonnadaraya:eq}, the soft labels are 
nonnegative of dimension $L$ and norm $||
\vecy_i||_1 =1$ when the provided labels have that form.

Since our kernels $\kernel$  are random variables that are 
instantiated in each simulation, there are two concievable choices
to working with them. 
The first, which we adopt,  is to compute a soft label in each 
simulation, and take the expectation. With kernel density we have
\begin{equation} \label{llabel}
\vecf_i = \E\left[\frac{\sum_{j\leq n_\ell} \kernel_{ij} m_j \vecy_j}{\sum_{j\leq n_\ell} m_j \kernel_{ij}}\right]\ . 
\end{equation}
The alternative is instead to estimate the expectation
$\overline{\kernel}_{ij}$ per entry and plug it in the respective
expression \eqref{kNN:eq} or \eqref{watsonnadaraya:eq}.
Our reasoning for prefering the former choice is 
preserving the dependencies when computing the density estimates in the relative location of seed nodes across simulations.
} 
 
\notinproc{
\section{Approximate soft labels} \label{sketch:sec}
In this section we start tackling the issue of highly scalable
computation of {\em approximate} soft labels.  We use
Monte Carlo simulations to estimate the expectation and 
sample-based sketches \cite{ECohen6f,ECohenADS:TKDE2015} to
approximate the soft label in each simulation.  

 Recall that the labels $\vecf_i$ are an expectation which we estimate
 using the average of $T$ independent draws, obtained via Monte Carlo
 simulations, of the soft-label random variable $f'_i$.
With kernel density, we estimate \eqref{llabel} as the average of 
$T$ draws of
\begin{equation} \label{onesimulation}
f'_i = \frac{\sum_{i\leq n_\ell} \kernel_{ji} m_i \vecy_i}{\sum_{i\leq n_\ell}
  m_i \kernel_{ji}}\ .
\end{equation}
We consider  the 
statistical guarantees we obtain for the average of $T$
independent  (exact) random variables $f'_i$ as an estimate of $\vecf_i$.
\begin{lemma}
With $T=\epsilon^{-2}$,  the average estimate of each component of $\vecf_i$ has absolute error bound that is {\em well concentrated} around $\epsilon$ (probability of absolute error that exceeds $c\epsilon$ is at most $2\exp(-2c^2)$). 
\end{lemma}
\begin{proof}
This is an immediate consequence of Hoeffding's inequality, noting
that entries of our label vectors are in $[0,1]$.
\end{proof}

 We next consider computing \eqref{onesimulation} for a single
 simulation.  As we stated in the introduction, exact computation for the nearest-seed estimator
 is simple.  With $k=1$,
the learned label is $f'_i = \vecy_j$, where $j = \arg\max_h \kernel_{ih}$.  The learned labels of all nodes can be computed very efficiently: For distance diffusion, we can use a single Dijkstra computation with the priority heap initialized with all seeds (find the closest seed to each node). For reach diffusion, we can similarly use a single survival threshold search (version of Algorithm~\ref{effthreshsketch:alg} without the pruning).

For kernel density, however, exact computation requires the values of the positions $N_{ij}$ for
all $i>n_\ell$ and $j\leq n_\ell$.
With distances, it is widely believed that there is no subquadratic
algorithm and even the representation alone is quadratic.
With reach diffusion, on undirected (symmetric) graphs, all pairs $t_{ij}$ can be
represented efficiently using  a single minimum weight spanning tree (MST)
computation on a graph with edge weights $1/\mu_e$.  The computation is near-linear in the number of edges.  
The graph cuts defined by the MST compactly specify $t_{ij}$ for all
pairs.   
Our interest here, however, is directed graphs, where the problem does
not seem much easier than shortest paths computations:
The computation
of $t_{ij}$ and $N_{ij}$ for one source node $i$ and all $j$ can be performed by a
graph search from $i$, but it seems that separate searches are
needed for different source nodes, similarly to the corresponding
problem with distances.  
Moreover, while $n_\ell$ searches suffices to compute $t_{ij}$, we seem to need $n_u \gg n_\ell$ searches to also compute $N_{ij}$.

 We approach this (for both reach and distance diffusions)  by using instead 
{\em estimates} $\hat{f}'_i$  of $f'_i$, which can be scalably computed
for all $i>n_\ell$.
 We then estimate $\vecf_i$ by averaging the $T$ estimates
$\hat{f}'_i$.  
Our estimates $\hat{f}'_i$ are obtained by computing two sets of sketches for all nodes $i$:
\begin{itemize}
\item
The first set of sketches is with respect to the full set of nodes, or more precisely, all nodes $h$ with $m_h>0$.  These sketches
are used to estimate the mass
$m(R_{\tau}(i))$  for all $i$ and for all $\tau$.  
\item
The second set of sketches is with respect to seed nodes.
They provide us, for each node $i$,  a
 small tailored weighted sample of seed nodes $S(i) \subset [n_\ell]$.
The sampling is such that the inclusion probability of $j$ is
proportional to $m(j)$ and inversely proportional to its position {\em
  in the seed set} when ordered by $t_{ij}$ ($d_{ij}$ for distances).
  For each $j\in S(i)$, the sketch also provides us with 
the exact value of $t_{ij}$ ($d_{ij}$ for distances) and a conditional inclusion probability $p_{ij}$.
\end{itemize}

  Using these sketches, we compute our per-simulation label estimate $\hat{f}'_i$ as follows.  
For each $i$ and $j\in S(i)$, we have $t_{ij}$ ($d_{ij}$ for distances), and use the first set of sketches to
compute the estimates
$$\hat{N}_{ij} \equiv \hat{m}(R_{t_{ij}}(i))\ .$$
For each $i$, we  use the sample $S(i)$ obtained in the second set of sketches to compute
\begin{equation}\label{esthat}
\hat{f}'_i = \frac{\sum_{j\in S(i)} \frac{1}{p_{ij}} m_j
  \hat{\kernel}_{ij}  \vecy_j}{\sum_{j \in S(i)}  \frac{1}{p_{ij}} m_j
  \hat{\kernel}_{ij}}\ ,
\end{equation}
where $\hat{\kernel}_{ij} = \alpha(\hat{N}_{ij})$.

\subsection{Sketches}

  The sketches we will use are MinHash and All-Distances Sketches (ADS), 
using state of 
the art optimal estimators \cite{ECohen6f,ECohenADS:TKDE2015,multiobjective:2015}.
To simplify and unify the
presentation, we use bottom-$k$ all-distances sketches
 \cite{ECohen6f,ECohenADS:TKDE2015,multiobjective:2015} 
for the two uses of sketches.  
The sketch parameter $k$ trades off sketch/sample 
  size and estimation quality. 
Note that other variations can also be used and the
 representation can be simplified when sketches are only used for
size estimation.
For further simplicity, we assume here that $m_i \in \{0,1\}$. See discussion in \cite{ECohenADS:TKDE2015} for the handling of general $m$.
We use the notation $U$ for the set of nodes that are being sketched,
which is the full set of nodes with positive mass for the first set of
sketches and only the seed nodes for the second set.

The sketches are randomized structures that are
defined with respect to a uniform
 random permutation $\pi$ of the sketched nodes $U$. We use the notation $\pi_j$ for the permutation position of $j\in U$. 
A bottom-$k$ MinHash sketch
 is defined for each $\tau$ and includes the $k$ nodes with minimum
 $\pi$ in the set $R_{\tau}(i) \cap U$.
The all-distances sketches $S(i)$ we work with
can be viewed as encoding  MinHash sketches of $R_{\tau}(i)\cap U$
for all values of $\tau$.
Formally,  
\begin{equation} \label{botksketch:eq}
j \in S(i) \iff \pi_j \leq \kth_\pi \{ h\in U \mid t_{ih} \geq 
t_{ij}\}\ . 
\end{equation}
With distances, the sketch is defined with the inequality reversed:
\begin{equation} \label{botksketchdist:eq}
j \in S(i) \iff \pi_j \leq \kth_\pi \{ h\in U \mid d_{ih} \leq 
d_{ij}\}\ . 
\end{equation}

The definition is almost identical for  reach diffusion and distance
diffusion.
Reach diffusion sketches are defined for survival times 
 $t_{ij}$, which are stronger for higher values,  whereas with 
 distances we use $d_{ij}$, which are stronger for lower values. 
To reduce redundancy, we will focus the presentation on  reach diffusion. 
To obtain the corresponding algorithms and sketches for distances, 
we need to reverse the inequality signs.

For each entry $j$ in the sketch $S(i)$, we also compute 
the {\em conditional inclusion probabilities} $p_{ij}$ of $j\in S(i)$.
In our context, we use these probabilities for the mass estimates 
obtained from the first set of sketches
and for the inverse probability estimate $\hat{f}'_i$ that use the second set of sketches.

The probability $p_{ij}$ is  defined with 
respect to (is conditioned on) 
the permutation $\pi$ on $U\setminus \{j\}$.  It is the
probability,
over the $|U|$ possible values of $\pi_j$ of having a value low enough
so that $j$ is included in  $S(i)$.
More precisely, for $j\in S(i)$, we consider the set of nodes
$$A_{ij} = \{ h\in U\setminus\{j\}  \mid t_{ih} \geq t_{ij}\}\  ,$$
which includes all nodes in $U$ other than $j$ that have survival
times at least $t_{ij}$. We then define 
\begin{equation} \label{pdef:eq}
p_{ij} = \left\{
\begin{array}{lr}
1 & \text{: if } |A_{ij}| < k \\
 \frac{\kth_\pi(A_{ij}) -1}{|U|} & \text{:  Otherwise}
\end{array}
\right.
\end{equation}
Where the operator $\kth_\pi$ returns the $k$th smallest permutation position 
of all elements in the set.  The node $j$ will always be included in the sketch if there are fewer than $k$ other nodes with a lower $t_{ih}$.  Otherwise, it will be included only if it has one of the lowest $k$ permutation positions among the nodes $U$, which means that it has a strictly lower permutation position than the $k$th position in $A_{ij}$.

Note that the set $A_{ij}$ is usually contained in $S(i)$, except for sometimes, when there are multiple elements $h$ with same
$t_{ih}$.  In this case it is possible for $p_{ij}$ to be defined by an element
not in $S(i)$.   We refer to such elements that are not included in
$S(i)$ but are used to compute inclusion probabilities for other nodes as $Z(i)$ nodes.

 We now explain how the sketches are used for the two tasks.
 For a node $i$, the sketch $S(i)$ can be viewed as a list of tuples of the form $(j,t_{ij},p(t_{ij}))$. 
When $U$ is the seed of seed nodes.  The second set of sketches is
computed with $U$ being the set of seed nodes.  In this case, 
the tuples $S(i)$  are the sample we use to compute the approximate
density estimates.
 The first set of sketches is computed with $U$ being
the set of all nodes with $m_i =1$.  We use this sketch to 
obtain {\em neighborhood estimation lists} which we use to obtain the estimates 
$\hat{m}(R_{\tau}(i))$. 
The neighborhood estimation list includes, for each represented $t$ value,
the entry $$(t,\sum_{h\in S(i) \mid t_{ih}\leq t} \frac{1}{p_{t_{ih}}})\ ,$$
 in sorted decreasing $t$ order.   This list can be computed by a
 linear pass over tuples $(j,t,p)$ in decreasing $t$ order.
To query the list with value $\tau$ we look for the last tuple in the list that has $t\geq \tau$ and return the associated estimate.

\subsection{Estimation Error Analysis}

  The estimation quality of $\hat{f}'_i$ \eqref{esthat} as an estimate of $f'_i$ \eqref{onesimulation} is
  affected by two sources of error.
The first is the quality of  the sample-based inverse probability estimate
\eqref{esthat} as an estimate of
\begin{equation} \label{simwithhat}
f^{(\hat{\kernel})}_i = \frac{\sum_{j \leq n_\ell}  \vecy_j
  \hat{\kernel}_{ij}}{\sum_{j \leq n_\ell}  \hat{\kernel}_{ij}}\ .
\end{equation}
The second is the quality of
$\hat{\kernel}_{ij}$
as an estimate of
$\kernel_{ij}$.

From the theory of MinHash and distance sketches, we obtain
  the following:
\begin{theorem} \label{estprop:thm}
For a sketch parameter $k$: 
\begin{itemize}
\item 
The expected size of the samples is bounded by 
$$\E[|S(i)\cup Z(i)|] \leq k \ln n_\ell$$ and the sizes are well concentrated. 
\item 
If $\hat{\kernel}_{ij}$ are 
nonincreasing in $t_{ij}$, then 
each component of the vector
$f^{(\hat{\kernel})}_i$
is estimated by \eqref{esthat} with mean square error (MSE) at most $1/k$ and good concentration. 
\end{itemize}
\end{theorem}
For the second source of error  we obtain:
\begin{lemma}
With sketch parameter $k$, 
the estimates $\hat{N}_{ij}$ are unbiased with Coefficient of
Variation (CV)
at most $1/\sqrt{2k}$ with good concentration. 
\end{lemma}
One caveat is our use of $\alpha(\hat{\kernel}_{ij})$ as an estimate of
$\alpha(\kernel_{ij})$.  Our estimates $\hat{\kernel}_{ij}$ have a small
relative error with good concentration, but for $\alpha(\kernel_{ij})$
to have this property we need it not to decay faster than 
polynomially.  More precisely, when $\frac{\alpha'(x) x}{\alpha(x)}
  \leq c$ then we obtain that the NRMSE is at most $c$ times that of
  the estimate $\hat{\kernel}_{ij}$.
In particular, when $\kernel_{ij} = 1/N_{ij}$, the estimates have NRMSE
at most $1/\sqrt{2k}$ with good concentration.

  We can now state overall worst-case statistical guarantees on our estimates of
  $\vecf_i$ as defined in \eqref{llabel}.  We use here the independence
  of our three sources  of error to slightly tighten the bound.
\begin{theorem} \label{sumupestquality}
   When using
$\epsilon^{-2}$ Monte Carlo
  simulations, and sketch parameter $k =
  \frac{1}{2}\epsilon^{-2}$ and when $\frac{\alpha'(x) x}{\alpha(x)}
  \leq 1$, then each component of $\vecf_i$
  is approximated with RMSE $\sqrt{3}\epsilon$ with good concentration.
\end{theorem}


\begin{algorithm}[h]\caption{Sketch survival thresholds \label{effthreshsketch:alg}}
{\small
\KwIn{$G=(V,E,\mu)$ a graph with nodes $V$, directed edges $|E|$, and lifetimes $\mu_e \geq 0$ for $e\in E$ ; Subset $U\subset V$ of nodes}
\KwOut{For $i\in V$, a sketch $S(i)$ of the set $\{(j,t_{ij}) \mid j\in U\}$}
\tcp{Initialization}
\ForEach{$i\in V$}{Initialize the sketch structure $S(i)$ \tcp*{Algorithm \ref{univMbyu:alg}}}
Compute a random permutation  $\pi:U \rightarrow |U|$ \;\\
\tcp{Main Loop:}
\ForEach{$j\in U$ in increasing $\pi_j$ order}
{Perform a pruned single-source survival threshold search from $j$ on the transposed graph\tcp*{Algorithm \ref{minthreshprunedsearch:alg}}}
\tcp{Finalize}
\ForEach{$i\in V$}{Finalize the sketch structure $S(i)$\tcp*{Algorithm \ref{univMbyu:alg}}}
}
\end{algorithm}

\ignore{
At each point, for each node $i$, the algorithms maintains two
structures:
\begin{itemize}
\item
A MinHash sketch of the set of all reachable nodes, that is used as a
distinct counter which maintains an estimate of the size
$|R_{\tau}(i)|$.  Any distinct counting structure can be used here
\cite{hyperloglog:2007} but for consistency with the second part we
work with a bottom-$k$ sketch with respect to the permutation $\pi$.
The sketch maintains the $k$ smallest permutation positions of nodes
in $R_{\tau}(i)$.
\item
A bottom-$k$ MinHash sketch structure of the set of reachable seed nodes.  This
is  used to select our sample of the seed nodes.  For a sketch parameter $k$, the sketch maintains the
$k$ smallest permutation ranks of reachable seed nodes.  A seed node
$j$ is sampled for $i$ when for some value of the parameter $\tau$,
$j$ is one of the $k$ smallest permutation orders among all seed nodes in
$R_{\tau}(i)$.
\end{itemize}
}
}
\notinproc{
\section{Algorithms for reach kernels} \label{algo:sec}

  We now consider the computation of the bottom-$k$ all-distances sketches.
These sketches were originally developed to be used with shortest-paths distances $d_{ij}$ and there are several algorithms and large scale implementations 
that can be used out of the box.  They compute the
  sketches or the more restricted application of 
neighborhood size estimates 
\cite{ECohen6f,reverseranks:sigmetrics2016,ECohenADS:TKDE2015,Akiba:KDD2016}.  
The different algorithms are designed for 
distributed node-centric, multi-core, and other settings. 
Most of these approaches can be easily adapted to 
estimate $m(R_{\tau}(i))$, when $m_i \in \{0,1\}$ (see discussion in \cite{ECohenADS:TKDE2015}) and there is a variation \cite{ECohenADS:TKDE2015} that is suitable for general $m$.
 The component of obtaining the sample and probabilities is more 
subtle, but uses the same computation (See \cite{ECohenADS:TKDE2015,multiobjective:2015}).

 For reach diffusion, we do not work with 
distances but with the survival thresholds $t_{ij}$.  As said, 
the sketches have the same definitions and form but we need to 
redesign the algorithms to compute the 
sketches with respect to thresholds.


 The sketching algorithm we present here for survival thresholds
builds on a sequential algorithm for ADS computation which is
based on performing pruned Dijkstra searches
\cite{ECohen6f,ECohenADS:TKDE2015}.  
The algorithm for distance sketching 
performs $O(|E|k\ln |U|)$ edge traversals and the total computation bound is $$O( n \log n + (|E|+n\log k) \ln |U|)\ ,$$ where $n$ is the total number of nodes.
The  algorithm has a parallel version designed 
to run on multi-core architectures \cite{reverseranks:sigmetrics2016}.

Our redesigned sketching algorithm for survival thresholds has
the same bounds.  Moreover, 
our redesign can be parallelized in the same way for multi-core architectures,
 but we do not provide the details here.

A high level pseudocode of our sketching algorithm for survival thresholds is provided as Algorithm~\ref{effthreshsketch:alg}.  We first initialize empty sketch structures $S(i)$ for all nodes $i$.  The algorithm builds the node sketches by processing nodes $j$
in increasing permutation rank $\pi_j$.  A pruned graph search is then performed from the node $j$. This search has the property that it visits all nodes $i$ where $j\in S(i)\cup Z(i)$.  The search updates the sketch for such nodes $i$ and proceeds through them.  The search is pruned when $j\not\in S(i)\cup Z(i)$.

 We now provide more details.  The first component of this algorithm is building the sketches $S(i)$.  The pseudocode provided as Algorithm~\ref{univMbyu:alg} builds on a state of the art design for computing universal monotone
multi-objective samples \cite{multiobjective:2015}.  The pseudocode includes the initialization, updates, and finalizing components of building the sketch for a single node $i$.  The structure is initially empty and then
tracks the set of pairs $(j, t_{ij})$ for the 
nodes $j$ processed so far that are members of the sketch $S(i)$.
 To build the sketch
 efficiently, the structure includes a min heap $H$ of size $k$ 
which contains the $k$ largest $t_{ij}$ values for processed $j\in S(i)$.  
 The structure is presented with updates of the form $(j,t_{ij})$, which are in increasing $\pi_j$ order.  A node $j$ is inserted to $S(i)$ when $t_{ij}$
is one of the $k$ largest $t_{ih}$ values of nodes $h$ already
selected for the sketch.  This is determined using the minimum priority in the 
heap $H$.  If
the node is inserted, the heap is updated by popping its min element and inserting $t_{ij}$.  The update can also result
in modifying the sketch in some cases when the node $j$ is not included in $S(i)$, but is in $Z(i)$, meaning that 
some inclusion probability of other node(s) is set to  $p(t_{ij})$.
To facilitate the computation of inclusion probabilities, we define
\begin{equation} \label{uprob}
u_j = \frac{\pi_j -1}{|U|}\ .
\end{equation}
 This is the probability for node with $\pi_h < \pi_j$ to have permutation rank smaller than $\pi_j$, when fixing the permutation order of all nodes except for $j$ and computing the probability conditioned on that.
We say that the 
update procedure for $(j,t_{ij})$ {\em modified the sketch}
if and only if $j \in S(i) \cup Z(i)$.

\begin{algorithm}[h]\caption{Maintain sketch $S(i)$, updates by increasing $\pi$ \label{univMbyu:alg}}
{\footnotesize
\tcp{Initialize:}
$S(i) \gets \perp$;  $p\gets \perp$ \\
$H \gets \perp$ \tcp*{min heap of size $k$, holding $k$ largest
  $(t_{ij}, -\pi_{ij})$   values processed so far (lex order)}
$prevt \gets \perp$ \tcp*{$t_{ij}$ of most recent $j$ popped from $H$}
\tcp{Process updates:}
\For{Update $(j,t_{ij},\pi_j)$, given by increasing $\pi_j$ order}
{
\If{$|H|<k$}{$S(i) \gets S(i) \cup \{j\}$;  Insert $j$ to $H$; Continue}
$y \gets \arg\min_{z\in H} (t_{iz},\pi_z)$ \tcp*{node with max $(t_{iy},\pi_y)$}
\eIf{$t_{ij} > t_{iy}$}
{$S(i) \gets S(i) \cup \{j\}$ \tcp*{Add $j$ to sample}
$prevt \gets t_{iy}$\\
\If{$p(prevt) = \perp$}{$p(prevt) \gets u_j$\tcp*{As defined in Eq. \eqref{uprob}}} 
Delete $y$ from $H$\\
Insert $j$ to $H$
}
(\tcp*[h]{$Z$ node check}) 
{
\If{$t_{ij} = t_{iy}$ {\bf and} $t_{iy} > prevt$}
{
 $p(t_{ij})\gets u_{j}$\\
 $prevt \gets t_{iy}$
}
}
}
\tcp{Finalize:}
\For(\tcp*[h]{keys with largest weights}){$x\in H$}
{
\If{$p(t_{ix}) = \perp$}{$p(t_{ix}) \gets 1$}
}
}
\end{algorithm}

The sketch of $i$ is computed correctly when the updates include all
nodes $j$ for which the sketch was modified: 
The computation of the sketch will not change if we
do not process entries $j$ that do not result in modifying the sketch.

We now describe the next component of the algorithm which is the pruned graph search from a node $j$. The searches are performed on the transposed graph, which has all edges reversed. Similar to the corresponding property of Dijkstra and distances, the search visits nodes $i$ in order of non-increasing $t_{ji}$.  The search is pruned at nodes where there were no updates to the sketch.
A pseudocode for the pruned search is provided as Algorithm~\ref{minthreshprunedsearch:alg}.  
The algorithm maintains a max heap that contains nodes $i$ that are
prioritized by lower bounds on $t_{ij}$.  The heap maintains the
property that the maximum priority $i$ has the exact $t_{ij}$.  The
heap is initialized with the node $j$ and priority $+\infty$. The
algorithm then repeats the following until the heap is empty. It
removes the maximum priority $i$ from the heap.  It then updates the
sketch of $i$ with $(j,t_{ij})$.  If the sketch was updated, all out
edges $e=(i,h)$ are processed as follows.  If $h$ is not on the heap,
it is placed there with priority $\min\{t_{ij},t_{e}\}$.   If $h$ is
in the heap, its priority is increased to the maximum of its current
priority and $\min\{t_{ij},t_{e}\}$.  
If the sketch of $i$ was not updated, the search is pruned at $i$ and out edges are not processed.
 For correctness, note that $\min\{t_{ij},t_{e}\}$ is trivially a
 lower bound on $t_{ih}$.

  We now need to establish that the sketches are still constructed correctly with the pruning:
\begin{lemma}
The search from $j$ reaches and processes all nodes $i$ for such that $j\in S(i)\cup Z(i)$.  When the node $i$ is processed, the update $(j,t_{ij})$ is with the correct survival threshold $t_{ij}$.
\end{lemma}
\begin{proof}
We show the claim by induction on permutation order.  Suppose the sketches are correctly populated until just before $j$.
Consider now a search from $j$ and a node $i$ such that $j\in S(i)$.
There must exist a path $P$
from $i$ to $j$ such that for any suffix $P'$ of the path from some
$h$ to $j$, $\min_{ e \in P'} \mu_e = t_{hj}$.

 We will show that the reverse search from $j$ can not be pruned in any of the nodes in $P$.
Therefore, $i$ must be inserted into the search heap and subsequently be processed. Assume to the contrary that the search is pruned at $h\in P$.
For the pruning to occur, there
 must be a set of nodes $Y\subset S(h)$ of size $|Y| \geq k$ 
such that $\pi_y < \pi_j$ and $t_{yj} \geq t_{hj}$.  
Let $P'' = P \setminus P'$ be the prefix of the path $P$ 
from $i$ to $h$ and let $T'' = \min_{ e \in P''} \mu_e$.  
Then by definition, for all $y\in Y$,
$$t_{iy} \geq \min\{T'',t_{yj}\} \geq \min\{T'',t_{hj}\} = t_{ij}\ .$$
Since there are at least $|Y| \geq k$ nodes with $\pi_y<\pi_j$ and $t_{iy}\geq t_{ij}$, this implies that $j \not\in S(i)$, and we obtain a contradiction.
A similar argument applies when $j\in Z(i)$.

Lastly, we need to argue that when node $i$ is removed from the heap
and processed, its priority is equal to $t_{ij}$.  It is easy to
verify that the heap maintains the property that the priorities are
lower bounds on survival thresholds.  This is because for any heap
priority, there must be path to $j$ with minimum $\mu_e$ equal to that priority.

 We need to show that equality holds when $i$ is processed.
The nodes $h\in P$ on the path are in
non-increasing order of $t_{ih}$.  Let $0=\tau_1>\tau_2>\cdots$ be the different survival threshold values on the path.  We prove this by induction on $\tau_i$.
are processed not necessarily in path order, but in non-increasing order of $t_{ih}$.
Initially the heap contains only $(j,\infty)$, which is the correct threshold.
Assume now it holds for all nodes with survival thresholds $\geq \tau_i$.
Consider now the path edge $e$ from a node $h$ with $t_{hj}=\tau_i$ to a node $h'$ with $t_{h'j}=\tau_{i+1}$.  This edge must have lifetime $\mu_e = \tau_{i+1}$.
When the node $h$ is processed, $h'$ is placed on the heap with
priority $\min\{\tau_i,\mu_e\} = \tau_{i+1}$, which is equal to
$t_{h'j}$.  If it was already on the heap, its priority is increased to $t_{h'j}$.  Consider now other path nodes $h''$ with $t_{h''j} = \tau_{i+1}$.  This nodes must be placed on the heap with the correct threshold when the previous path node is processed (it is possible for them to be placed with the correct priority also before that).  Therefore, all path nodes with  $t_{hj} = \tau_{i+1}$ will be processed with the correct priority.

\end{proof}

\begin{algorithm}[h]\caption{Pruned single-source survival threshold search \label{minthreshprunedsearch:alg}}
{\footnotesize
\KwIn{Source node $j$}
\tcp{Initialization:}
$H \gets \perp$ \tcp*{Empty max heap of nodes $i$. Priority is a lower bound on $t_{ij}$}
Put $(j,+\infty)$ in $H$ \tcp*{$j$ with priority $t_{jj} = +\infty$}
\tcp{Main loop:}
\While{Heap $H$ not empty}
{
 Pop maximum priority $(i,t_{ij})$ from $H$\\
 Update the sketch $S(i)$ with $(j,t_{ij})$ \tcp*{Algorithm \ref{univMbyu:alg}}
 \If{update modified sketch}{
 \ForEach{out edge $e=(i,h)$}{
 \eIf{$h\not\in H$}{Insert $(h,\max\{\mu_e,t_{ij}\})$ to $H$}{Update priority of $h$ in $H$ to the maximum of current priority and $\min\{\mu_e,t_{ij}\}$}
 }
}
}
}
\end{algorithm}

 We can now bound the computation performed by the algorithm.  
\begin{lemma}
The sketching algorithm performs in expectation at most $|E| k \ln |U|$ edge traversals.  The total computation is $$O(n\log n + (|E|+n\log k) k \ln |U|)\ $$ where $n$ is the total number of nodes.
\end{lemma}
\begin{proof}
The number of times a node is processed  by a pruned search (meaning that its out edges are processed) is equal to the number of times its sketch is modified, which is the size of the sketch.  From the analysis of distance sketches, we have a bound on the number of visits.  We obtain a bound of $|E| k \ln |U|$ on the number of edge traversals performed by the algorithm.  The other summand is due to heap operations when updating the sketches and in the pruned searches.
\end{proof}

\section{Parameter setting} \label{hyperpar:sec}

 Our models have several hyper parameters:  With reach diffusion,
 the selection of the lifetime random variables 
and possible dependencies between them.  With distance
 diffusion, 
the selection of the length random variables.   Another important
choice is the decay function $\alpha$ which converts distances or
ranks to affinity values.
Note that the same set of sketches supports the computation of labels
with respect to all non-increasing $\alpha$, so the tuning of this
hyper parameter is computationally cheap.

  Our algorithms and use of sketches provide us with
leave-one-out learned labels:  Specifically, 
for each member $i$ of the seed set $U$, and for each $\alpha$, 
we can compute a learned
label $\vecf'^{(\alpha)}_i$ with respect to seeds $U\setminus\{i\}$. 
This computation utilizes the same kernel density formula, summing
over the sample $S(i)$ with $i$ itself omitted.
The leave-one-out labels can be used
to learn a non-increasing $\alpha$ which 
minimizes the cost 
$$\min_{\alpha} \sum_{i\in U} ||\vecf'^{(\alpha)}_i - \vecy_i ||_2\ .$$
In this case, the seed nodes are used as training examples to learn
the kernel weighting.

  A separate question is obtaining
  class predictions (``hard'' labels) for unlabeled nodes.  
The simplest approach is to interpret the learned soft label as a probability
vector over classes.  More generally, the 
soft label can be interpreted as a
  signal collected from graph relations and the
pairs $$\{(\vecf'_i,\vecy_i) \mid i\in 
U \}$$  are used to train a model that predicts classes
from learned labels. 
 In our experiments we observed  that class predictions obtained using
logistic regression outperformed the naive approach of using
the largest entry in $\vecf_i$ to predict the class $\vecy_i$.



\ignore{
 Note that the selection of a decay function $\alpha$ can
 utilize the same sets of simulations and sketches.  Each evaluation for
 different
component length/lifetime  functions, however,  requires a fresh set  of
simulations and sketches.
}

} 
\section{Experiments} \label{exper:sec}

\onlyinproc{We perform experiments using the Movielens 1M
dataset.  Additional details and experiments are
  provided in the full version.}\notinproc{
 We performed experiments using the Movielens 1M
 \cite{movielen1m} and political blogs \cite{AdamicPB:2005} datasets
 (see Appendix).}  Our aim is two fold.  First, 
to evaluate the quality of reach and distance diffusion kernels
in a semi-supervised learning context.  
Second, to
 demonstrate a use case for our models and the selection of
length or lifetime variables.
Our evaluation here is not meant to assess scalability, as 
there are several highly scalable implementation of our ingredients: Shortest path
searches and 
distance and reachability sketching 
\cite{binaryinfluence:CIKM2014,GarimellaMGS:2015,timedinfluence:2015,reverseranks:sigmetrics2016,Akiba:KDD2016}.
We implemented the algorithms 
in Python and performed the experiments on a Macbook Air and a Linux workstation.

\notinproc{

\subsection{Movielens 1M data}}
 The data consists of about 1 million rating by 6,040 users of 3,952 movies.
  Each movie is a member of one or more of  18 genres\notinproc{:
	Action,  Adventure,  Animation, Children's, Comedy, Crime,
        Documentary,  Drama, Fantasy, Film-Noir, Horror, Musical,
        Mystery, Romance, Sci-Fi, Thriller,  War,  Western}.
\notinproc{  
}Our examples $M$ are the $\approx$3.7K movies with both listed genres and ratings.\notinproc{Respectively
51\%,35\%,11\%,3\%,0.5\%,0.03\%  of these movies have exactly 1 to 6 genres.}
We represented the ``true'' label $\vecy_m$ of a movie $m$ with $c$ listed genres as an $L=18$
 dimensional vector with weight $1/c$ on each listed genre and weight
 $0$ otherwise.  The (weighted) occurrence of genres is highly skewed and varies from 30\% to 0.6\% of the movies.
Note that the provided labels and also our learned labels 
  have the form of probability vectors over genres.
We  use the notation $\Gamma(m)$ for the set of users that rated movie $m$ and by 
 $\Gamma(u)$ for the set of movies rated by $u$. 

 We build a graph with a node for each movie and each user.
For each  user $u$ and movie $m\in \Gamma(u)$, we place two directed edges, $(m,u)$
and $(u,m)$ (We do not use the numeric ratings provided in
the data set, and only consider presence of a rating).  
We evaluated performance for a small set of length/lifetime and kernel weighting schemes, without attempting to optimize the choice.
For kernel weighting we used $\alpha(x) \in \{1/x, 1/x^{1.5}, nn\}$, where $nn$ is the ``closest seed.''
The length/lifetime
schemes are listed in Table~\ref{length:table} and are as follows:

\paragraph*{Reach diffusion lifetimes}
We  tune the amount in which paths through high degree 
nodes are discounted by choosing a non-decreasing function $g(x)$.
The lifetime of each user 
  to movie edge $e = (u,m)$ is an independent exponential random variables $\mu_e \sim 
  \Exp[g(|\Gamma(u)|)]$.  All movie to user edges $e=(m,u)$ have fixed
  $+\infty$ lifetimes. Finally, each movie node $m$ has an independent
  ``pass through''   lifetime  $\mu_{mm} \sim \Exp[g(|\Gamma(m)|)]$. 

\paragraph*{Distance diffusion lengths}
Here we use a non-increasing $g(x)$ to tune the discounting of paths through
high degree nodes. We use
fixed-length schemes with lengths $\ell_{vw} = g(|\Gamma(v)|)$ for all
edges. 
Our randomized schemes are also specified using an offset value
$\delta \geq 0$, which tunes the penalty for paths with additional hops.
The length of each user
  to movie edge $e = (u,m)$ is an independent exponential random variable $\ell_e \sim
  \Exp[g(|\Gamma(u)|)]$. Each movie to user edge $e=(m,u)$ has length $\ell_e = 0$.
Finally, with each movie $m$, we associate a pass-through length
that is $\ell_m \sim \delta+ \Exp[g(|\Gamma(m)|)]$.
With $g(x) = 1/x$, we have the property
that the shortest out edge from a node $v$ has 
length distribution $\Exp[1]$ regardless of $|\Gamma(v)|$. Functions
that decay more slowly give more significance to higher degree nodes.

This randomized length scheme has
a compelling interpretation:  For a movie $m$, the order of  2-hop movies sorted by 
increasing distance from $m$
has the same distribution as sequential
weighted sampling without replacement of movies according to the
similarity of their users, when similarity is defined as: 
$${\text sim}(m,m') = \sum_{u \in \Gamma(m)\cap \Gamma(m')}
g(|\Gamma(u)|)\ .$$
In particular, the closest seed in each simulation (used in our nn
weighting) is a weighted
sample according to this similarity measure. 
  With $g(x) = 1/\log(x)$ we obtain the
Adamic-Adar similarity \cite{AdamicAdar:2005} popular in social
network analysis.   Note that our model captures these pairwise
movie-movie relations while working with the original user-movie
interactions,
without explicit computation or approximation of these similarities.
Beyond 2-hops, the distance order depends on the complex path ensembles
connecting movies to $m$, and
their interactions, but have desirable intuitive properties:
movies $m'$ with ``stronger'' connectivity ensembles are in
expectation closer and movies $m'$ and $m''$ with different strengths
and highly dependent ensembles will have the stronger connection consistently closer.
 The use of pass-through lengths with movie nodes and $0$ lengths for $(m,u)$ edges is equivalent to using pass-through lengths of $0$ and 
using for all outgoing $e=(m,u)$ edges identical lengths $\ell_m$.
The effect of independent lengths of outgoing edges rewards multiple paths even when
they traverse the same node whereas the use of same (random) lengths
rewards only node-disjoint paths.


\ignore{
The purpose of the offset parameter $\delta$ is to control
the ``cost'' of additional hops: With a very high offset, a movie
$m'$ with a  $4$ hop path to $m$ would always be farther than a movie $m''$
$2$ hop away.  With low offset,  when the path ensemble from $m$ to
$m''$ contains many independent paths through low degree nodes, and the ensemble from $m$ to
$m'$ is sparse and involves one or few very high degree users, then
$m''$ would likely to be closer than $m'$.  
}

\begin{table*}
\caption{Lengths and lifetime schemes for
  Movielens1M\label{length:table}}
\centering
{\scriptsize
\begin{tabular}{l|l|l}
Scheme name & specifications & parameters \\
\hline
Dist Exp$[g(x)]+\delta$ &  $\ell_{um} \sim \Exp[g(|\Gamma(u)|)]$ & $g(x) = \frac{1}{x}$, $\delta\in \{0,50,200\}$ \\
 & $\ell_{mm}
   \sim \Exp[g(|\Gamma(m)|)]+ \delta$ & 
 $g(x) = \frac{1}{\sqrt{x}}$ , $\delta=50$ \\ 
\hline
 Dist $g(x)$ (Fixed-length) & $\ell_{vw} = g(|\Gamma(v)|)$ & 
 $g(x) = \{1, \log_2(1+x), \sqrt{x}, x \}$ \\
 \hline
 Dist ExpInd$[\frac{1}{g(x)}]+\delta$ &  $\ell_{um} \sim \Exp[g(|\Gamma(u)|)]$ & $g(x) = 1/x$,  $\delta=50$ \\
 &   $\ell_{mu} \sim \Exp[g(|\Gamma(m)|)]+ \delta$ &  \\
 \hline
 Reach Exp$[g(x)]$ & $\mu_{um}  \sim \Exp[g(|\Gamma(u)|)]$ & $g(x) = \{x, \sqrt{x}\}$ \\
 &    $\mu_{mm}  \sim \Exp[g(|\Gamma(m)|)]$ & 
\end{tabular}
}
\end{table*}

\ignore{
 The specific lengths schemes we evaluated are listed in Table~\ref{length:table}:
(i)~Randomized schemes with $g_1 \equiv g_2 \equiv g$ and different
offsets.
(ii)~A randomized  scheme that sets the lengths 
of $(m,u)$ edges {\em independently} as $\Exp[g(|\Gamma(m)|)]+\delta$
instead of using the same pass-through length for $m$.
(iii)~Fixed lengths schemes where the length of an edge $(v,w)$ is a
fixed non-decreasing function of $|\Gamma(v)|$.
}

\paragraph*{Computation}
We performed multiple Monte Carlo simulations of each model. 
In each simulation we obtain a fresh 
set of edge lengths/lifetimes.  The closest seed (nn) weighting
requires computation equivalent to a single graph search and
the learned label of $i$ is the label of that
closest seed.
With the other kernel weights we compute two sets of sketches\notinproc{ as
outlined in Section~\ref{hyperpar:sec}, both with sketch parameter $k=16$,} and compute estimates using the sketches.
Our final learned labels $\vecf_i$
are the average of the output of the different simulations.
We also 
  computed learned labels for seed movies, based only 
on the labels of other seed movies. 
We expect our quality to improve with the sketch
parameter $k$ (which controls the 
quality of the estimates obtained from the sketches) and  with
the number of simulations. 
In our experiments we performed up to 200 simulations with the nn weighting and up to 50 simulations with the schemes that require sketches.

\paragraph*{Spectral methods}
For comparison, we implemented a popular symmetric
spectral method of label learning.  As discussed in the introduction, there are many 
different variations.  We chose 
to use the {\sc EXPANDER} formulation with the Jacobi iterations as 
outlined in \cite{RaviDiao:AISTATS2016}. 
{\sc EXPANDER}  initializes the labels $\vecf^{(0)}_i$ of node $i$ to the seed label 
$\vecf^{(0)}_i = \vecy_i$ when $i\leq n_\ell$ and to the uniform prior 
$\vecf^{(0)}_i = \vecu$ otherwise. 
The labels are then iteratively updated using 
\begin{equation} \label{expanderiter:eq}
f^{(t+1)}_i = \frac{\mu_1 I_{i\leq n_\ell} \vecy_i + \mu_2 \sum_{j\in 
    \Gamma(i)} w_{ij} f^{(t)}_j + \mu_3 \vecu}{\mu_1 I_{i\leq n_\ell} +
  \mu_2\sum_{j\in\Gamma(i)}w_{ij} + \mu_3}\ . 
\end{equation}

We constructed a graph from the Movielens1M dataset as described 
above, with a node for each movie or user and an edge for each rating. 
We used the same weighting parameters $\mu_1=1$,
$\mu_2=0.01$, $\mu_3=0.01$ as in the experiments in \cite{RaviDiao:AISTATS2016}. 
We used uniform relative weights of neighboring nodes, which were 
either all $w_{ij}=1$ or the inverse of the degree $w_{ij}=1/|\Gamma(i)|$. 
We performed up to two hundred iterations. 

 The uniform prior used in \cite{RaviDiao:AISTATS2016} resulted in poor quality learned labels.  
We speculated that this is because our seed labels (and data set 
  labels) are very skewed (some genres are much more common 
  than others) 
  whereas the experiments in \cite{RaviDiao:AISTATS2016} selected 
  balanced seed sets. 
We tried to correct this by instead using a prior that is 
equal to the average seed label.  This prior was used both in 
initialization and in the propagation rule. 

 With the average prior, with uniform weighting of $w_{ij}$ the 
 learned labels did not converge and also did not improve with 
 iterations.  With inverse of the degree weighting, the learned labels 
stabilized in fewer than 20 iterations.

\paragraph*{Seed sets}
  Our seed sets $S$ are subsets of $M$ selected uniformly at random.  We use 
  seed sets of sizes 
$s \in \{20, 50, 100, 200, 500\}$
which roughly correspond to $0.5\%$ to 
$12\%$ of all movies in $M$. We selected 5 random permutations of the
examples $M$. The sets of seeds were prefixes of the same
permutation and the test set was the suffix of movies not selected for any seed set.

\paragraph*{Quality measures}
We use both the average square error (ASE) and other metrics that directly evaluate the 
effectiveness of the learned label in predicting genres (classes).
The squared error of $\vecf_i$ with respect to the 
true label $\vecy_i$ is defined as
$|| \vecf_i -\vecy_i ||^2_2 = \sum_{j\in [L]} (f_{ij}-y_{ij})^2\ .$
Note that the sum $\sum_i ||\vecf-\vecy_i||^2_2$ is minimized by the
average of $\vecy_i$.
Our baseline quality is the average of $||\overline{\vecy}(S) -\vecy_i||^2_2$, where $\overline{\vecy}(S)$ is the
average seed label
\begin{equation} \label{avelabel:eq}
\overline{\vecy}(S) \equiv \frac{1}{|S|}\sum_{i\in S} \vecy_i\ .
\end{equation}

To predict classes, we use the learned label $\vecf_i$ to compute an importance order
of classes (we explain below how such an order is obtained).  
We then compute a {\em success score} in $[0,1]$ for
the order as follows, using
the true label $\vecy_i$:
Each position $j$ in the order with $y_{ij}>0$ contributes $1/j$ to
the numerator of the success score.  The success is then normalized by $H_r = \sum_{i=1}^r 1/r$
for a movie with $r$ listed genres. For example,  a movie with $r$ genres that are
the first $r$ positions in the order gets success score of $1$.  A
movie with one genre that is in the $j$th position in the order gets a
success score of $1/j$.  A movie with two genres that are in positions
2 and 3 of the order gets a success score $\approx 0.56$.

 We evaluated three methods of ordering classes $j$, using a decreasing order according to the following.   {\em
   Mag}:  $f_{ij}$ learned label entry; {\em rMag}: $\frac{f_{ij}^2}{\overline{y}(S)_j^2}$
penalizes entries lower than the corresponding average seed entry;
 {\em LoReg}: Sort by order of decreasing probabilities of $i$ having
 the label $j$ ($y_{ij}>0$) given $f_{ij}$.  The probabilities are
 computed using (regularized) logistic regression models. We used this method only
 with diffusion models, as they support efficient computation of learned 
labels of seed nodes, based only on other seed nodes.
Specifically, 
for each genre $j$, for each seed $i$, we used $f_{ij}$ as a positive examples 
when  $y_{ij}>0$ and as a negative example when 
$y_{ij}=0$.

\paragraph{Results and discussion}
Some representative results for the 
average square error and the success scores of
selected schemes are provided in Figure~\ref{MSEmovielen:fig}.
The figures showing success scores also show a 
baseline success of using a decreasing order using the average seed
label.  This baseline already achieves average success score of 0.55.
This is because the high skew of the class distribution.

As expected, the quality of the learned labels improves with the number of seeds.
We can see that the diffusion-based methods  outperformed the label
propagation method.  In our settings, the LP learned labels converged
to vectors that are very close to the average seed labels, and the
quality measures we used did not separate them.
As for success scoring orders, rMag consistently improved over Mag
(not shown).  Both rMag and LoReg improved significantly over the baseline, with
rMag performing better on smaller seed set and LoReg performing
better for deterministic length schemes.
LoReg was less effective on 
smaller size seed sets because there 
were very few examples to work with.

The settings of $\alpha(x)=1/x,1/x^{1.5}$ performed similarly\notinproc{ and we show results only for $1/x^{1.5}$}.
The closest seed (nn) kernel requires more simulations to reach its
peak quality,
but note that simulations are considerably faster. The nn kernel
performed very well with the randomized distance schemes but poorly with the deterministic schemes.  This is
because the deterministic schemes do not improve with simulation and
the nn kernel uses essentially a single (closest) seed.  
The randomized lengths schemes outperformed the deterministic ones, and more significantly on smaller
seed sets.
\notinproc{

} We also noted the following.
The settings that performed best were $g(x)=1/x$ and $\delta=50$ for
randomized dist diffusion, $g(x)=x$ with reach diffusion,
and $1/\log_2(1+x)$ for fixed-length distance diffusion.
Quality was not sensitive to small variations in parameters.
The ExpInd schemes (with independent $(m,u)$ lengths) performed
somewhat worse than the basic scheme. 
Overall, the randomized distance diffusion schemes were the most effective.

\begin{figure*}[th]
\centering 
\notinproc{
\includegraphics[width=0.32\textwidth]{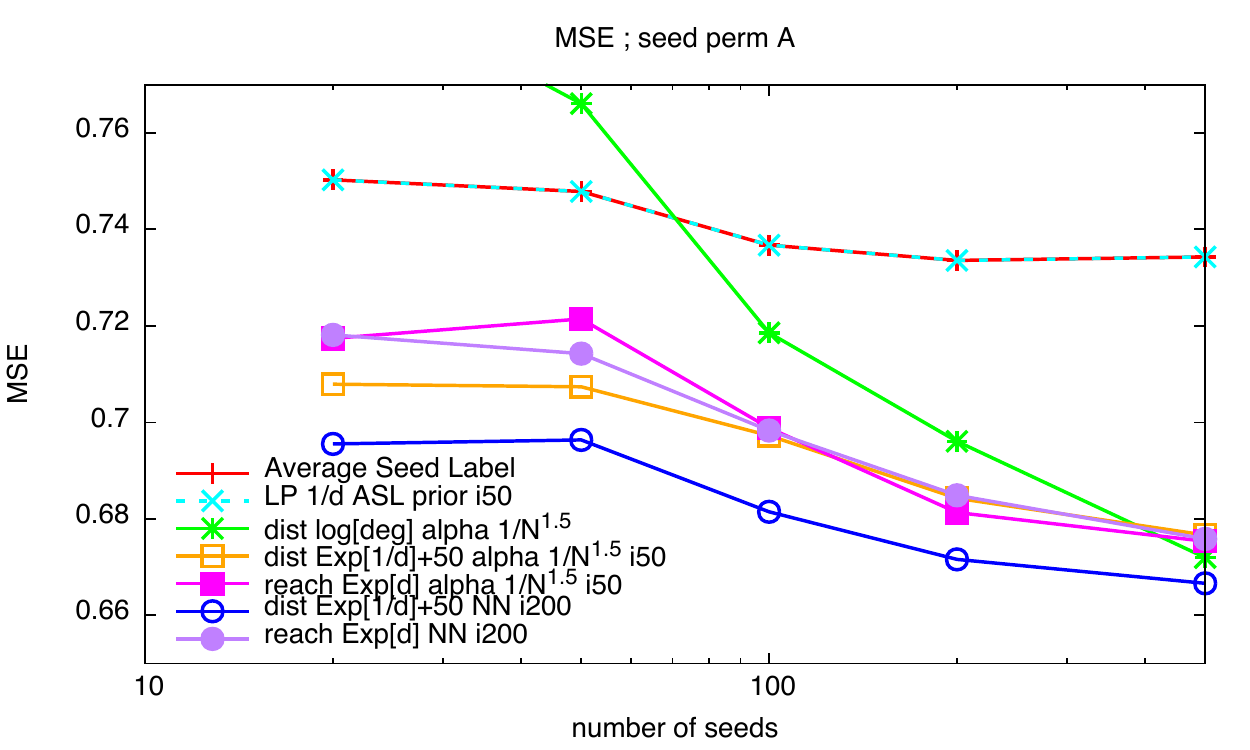}%
\includegraphics[width=0.32\textwidth]{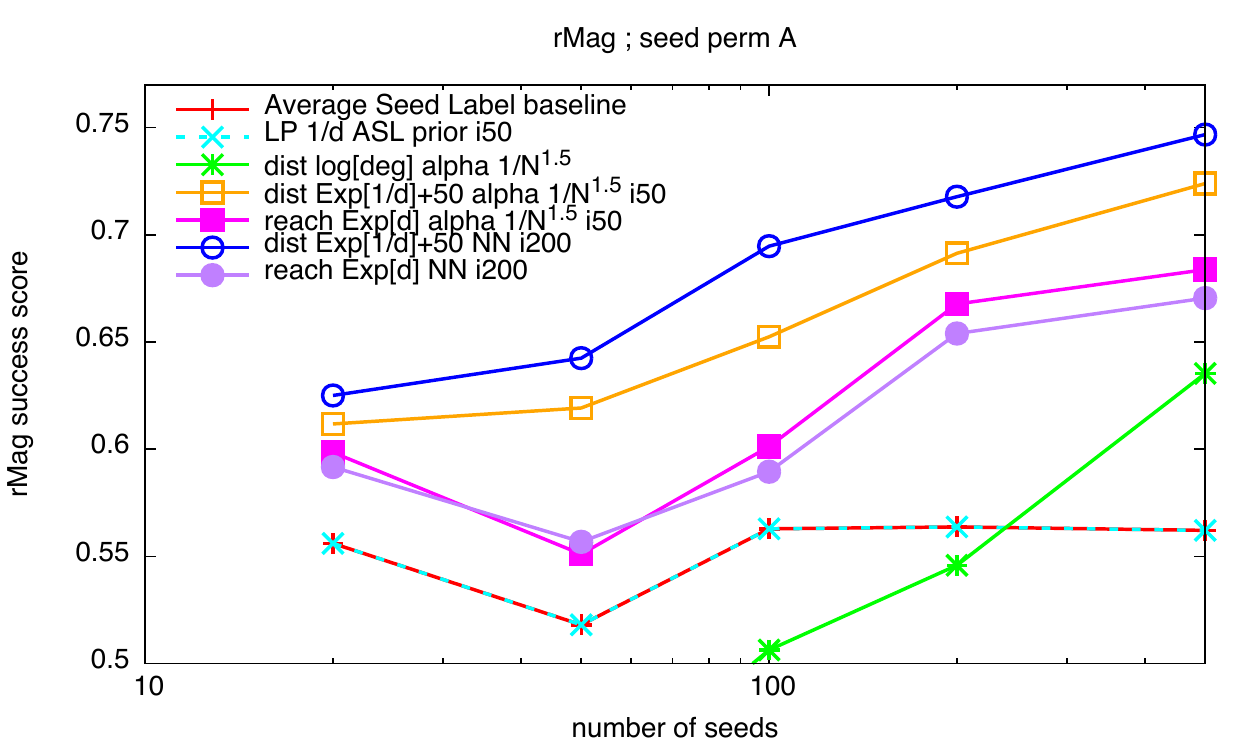}%
\includegraphics[width=0.32\textwidth]{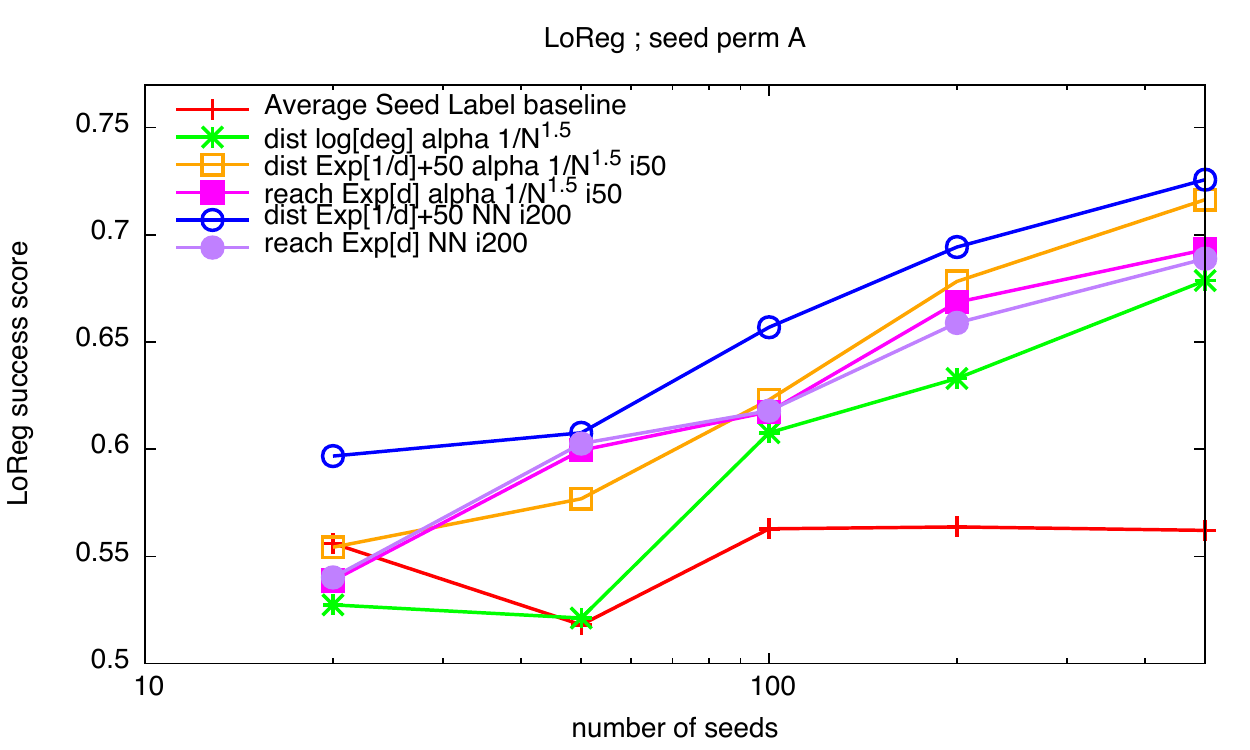} \\}
\includegraphics[width=0.32\textwidth]{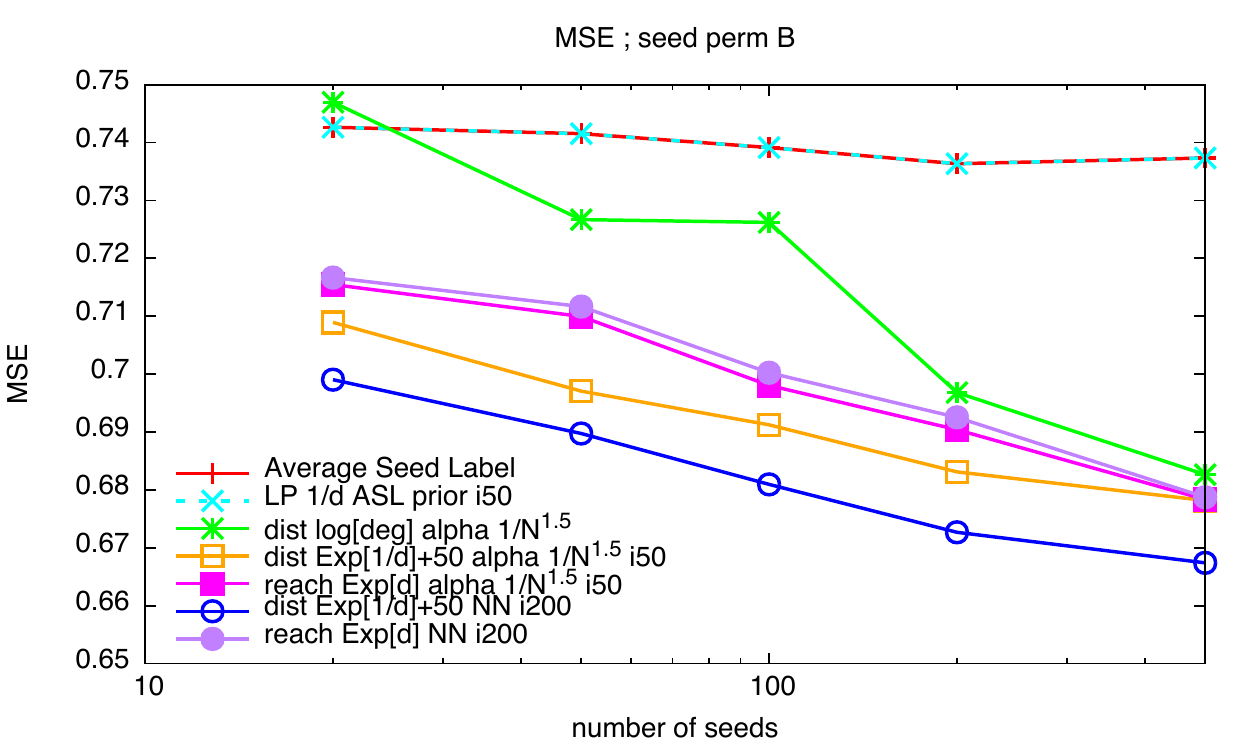}%
\includegraphics[width=0.32\textwidth]{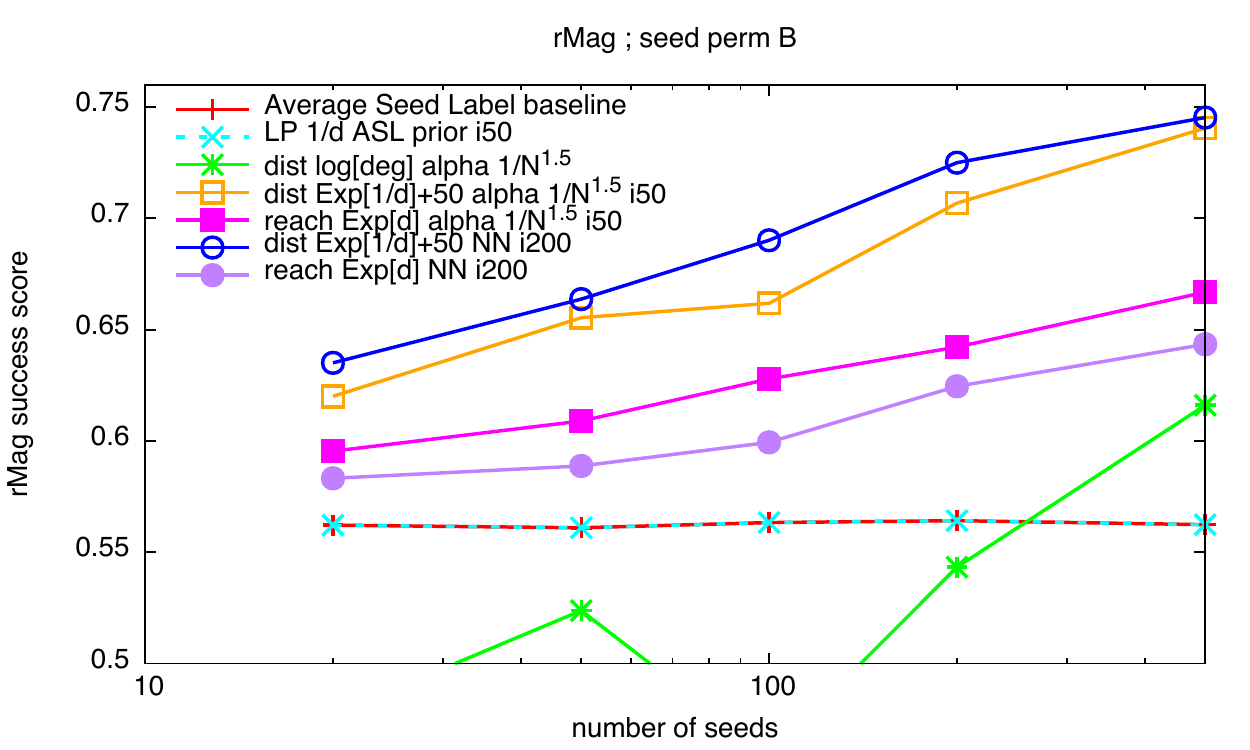}%
\includegraphics[width=0.32\textwidth]{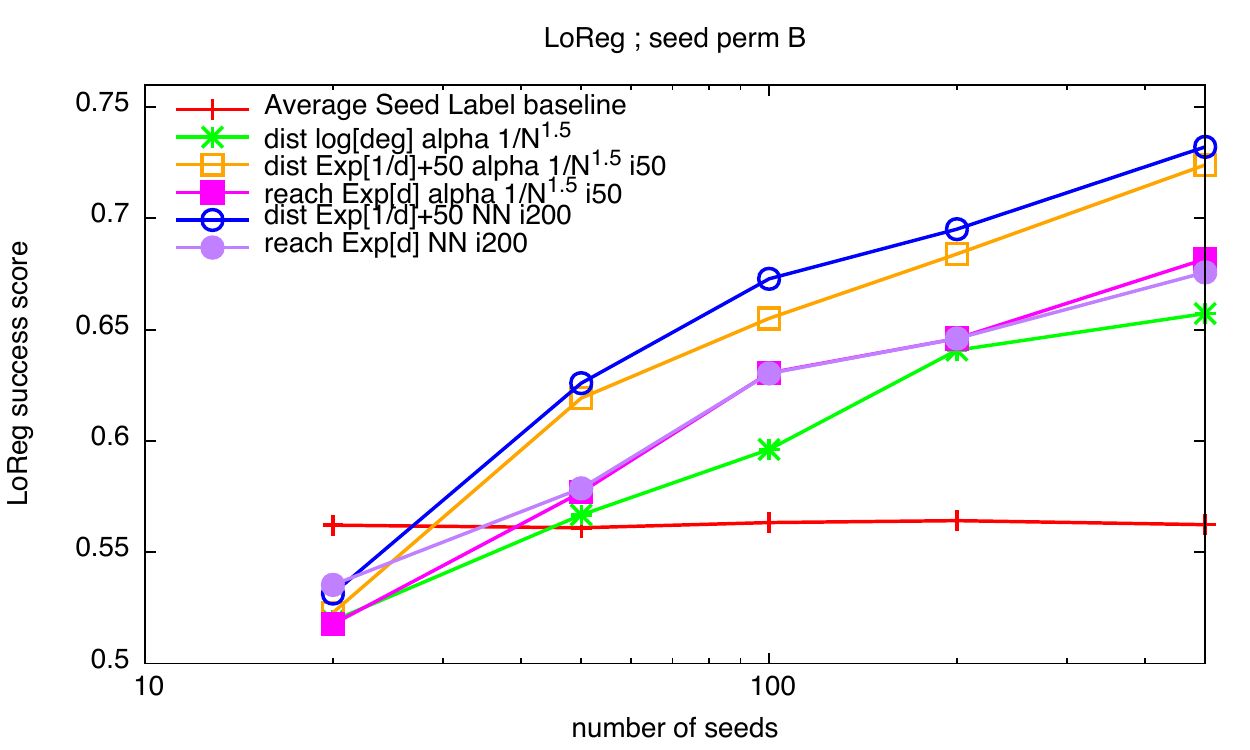} \notinproc{\\
\includegraphics[width=0.32\textwidth]{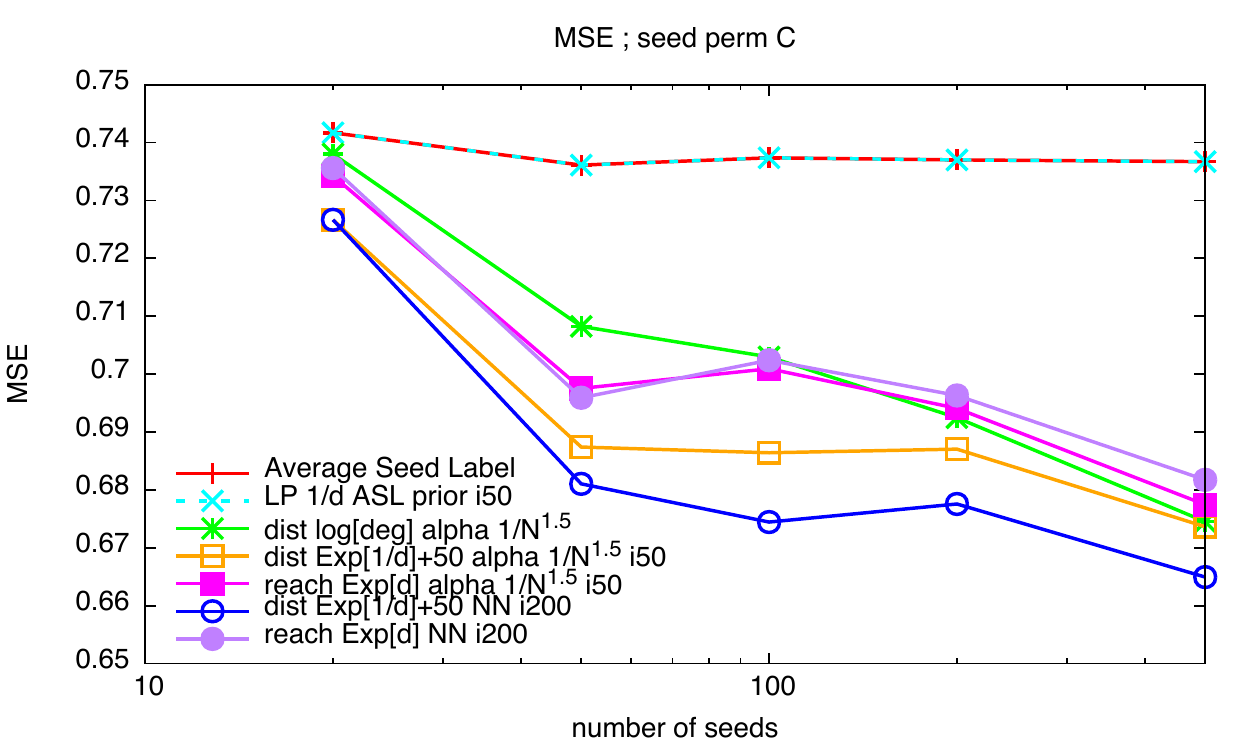}%
\includegraphics[width=0.32\textwidth]{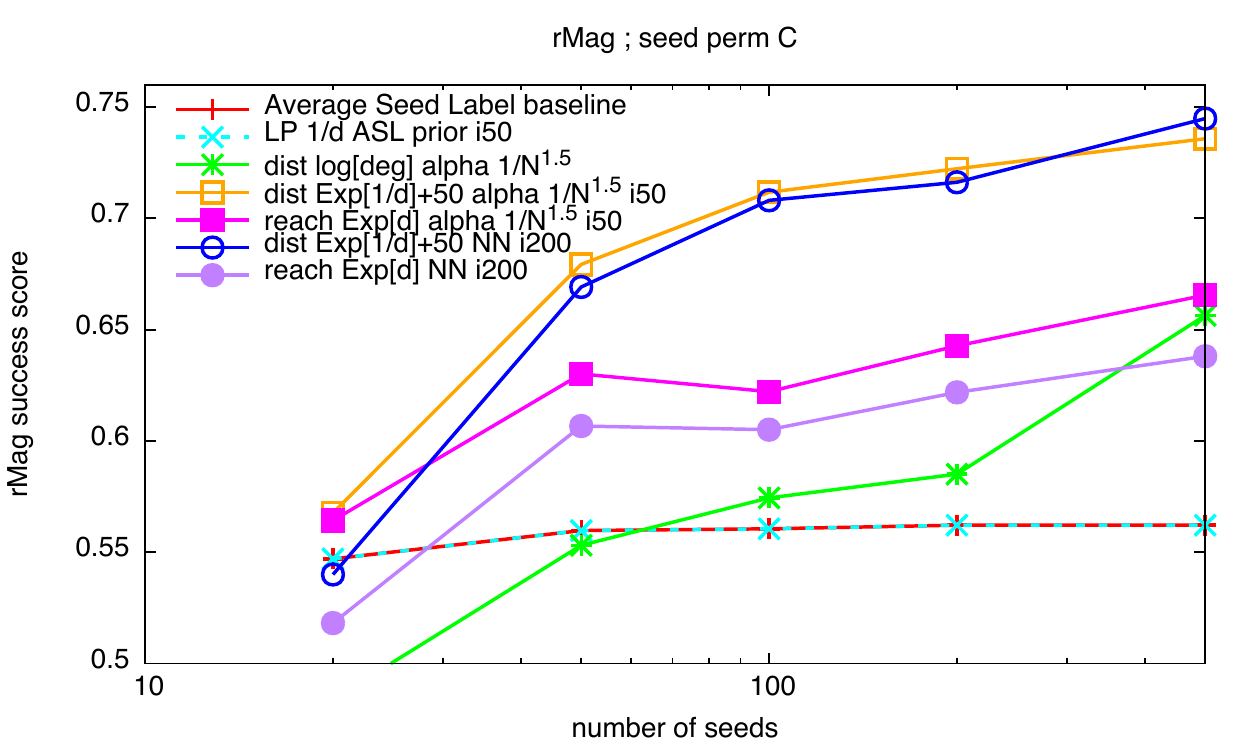}%
\includegraphics[width=0.32\textwidth]{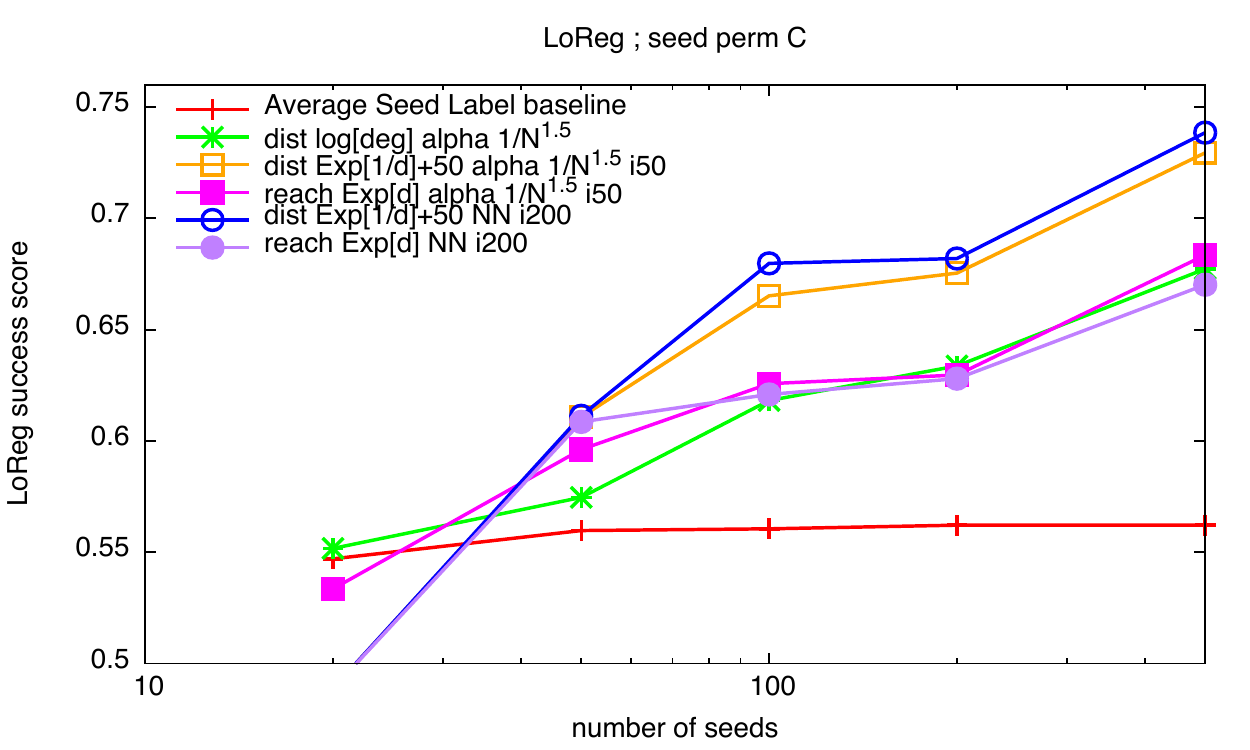}%
}
\caption{\label{MSEmovielen:fig}{\instance{movielens1M}: Average square error (lower is better) and rMag/LoReg success scores (higher is better) for selected schemes.}}
\end{figure*}


  To study performance in more detail we obtain
precision recall (PR) tradeoffs for our learned labels using
the prediction {\em margin}, which we
define as the 2-norm of the difference  between the learned
label and the average label 
$\Delta_i = || \overline{\vecy}(S) - \vecf_i||_2$.  
We then sweep a threshold value $\tau$.
The recall for $\tau$ is the fraction of examples (movies) $i$ for
which $\Delta_i \geq \tau$.
The precision is then defined as the average success score of these
examples.
Figure~\ref{sweepitermovielen:fig} shows the PR tradeoffs by sweeping the number of
simulations.  
We can see that with all schemes we obtain significantly higher
quality classifications with higher margin. This is important because in many applications we are interested in identifying the higher quality labels.  As for the effect of  simulations, the randomized schemes improve 
significantly with simulations, which shows the value of randomized 
lengths/lifetimes models.  
\notinproc{
Simulations can improve the deterministic schemes which use sketches,
due to use of sketch-based estimates, but the improvement is very limited.}

\begin{figure*}
\centering 
\includegraphics[width=0.32\textwidth]{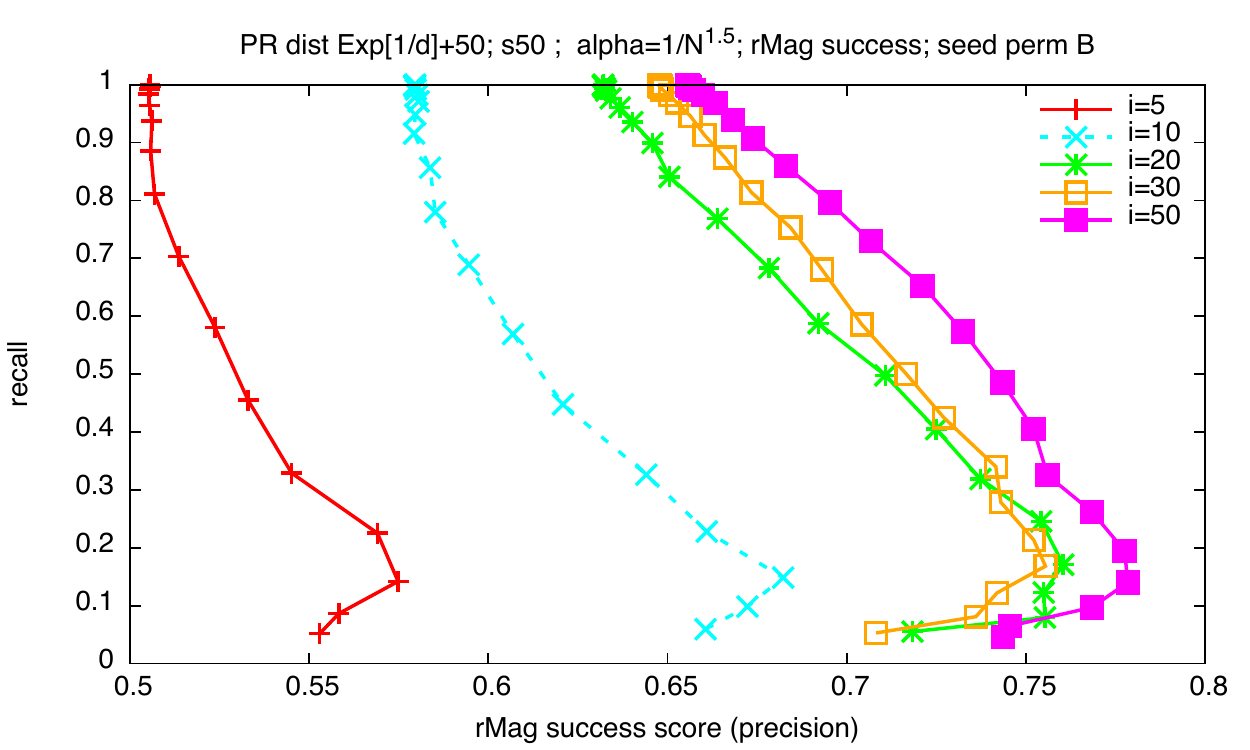}%
\includegraphics[width=0.32\textwidth]{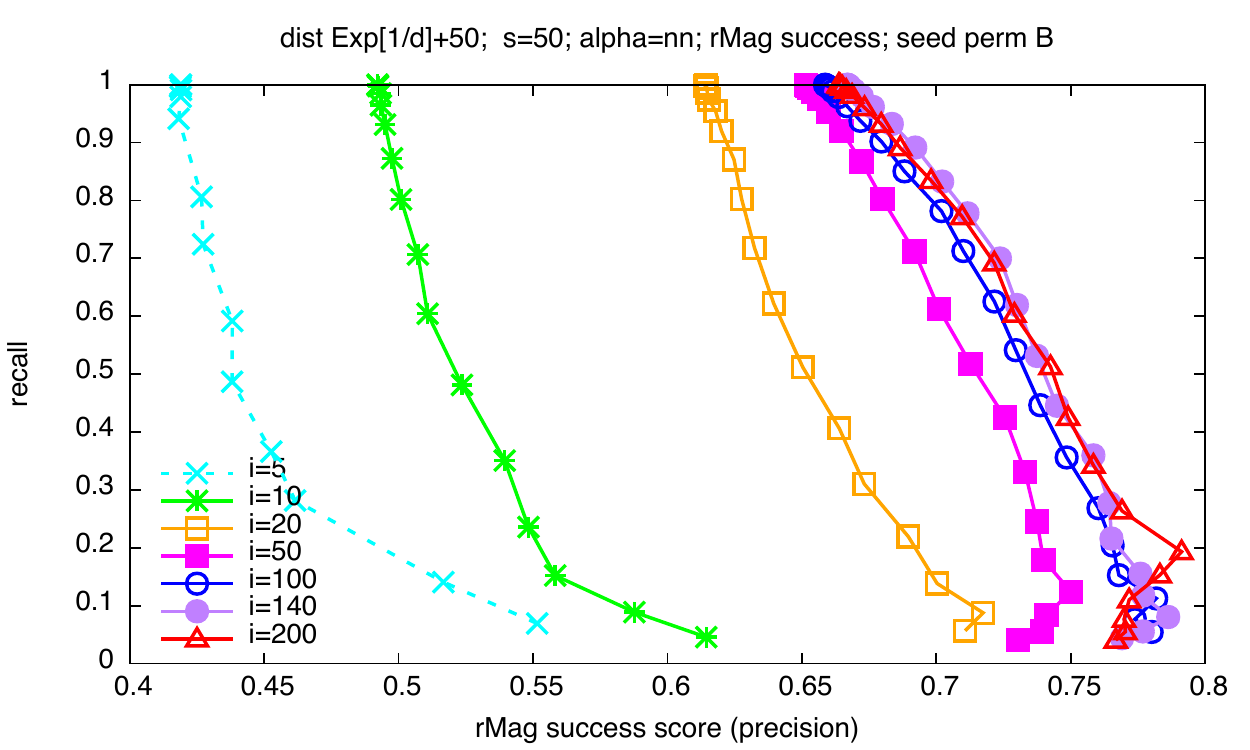}%
\includegraphics[width=0.32\textwidth]{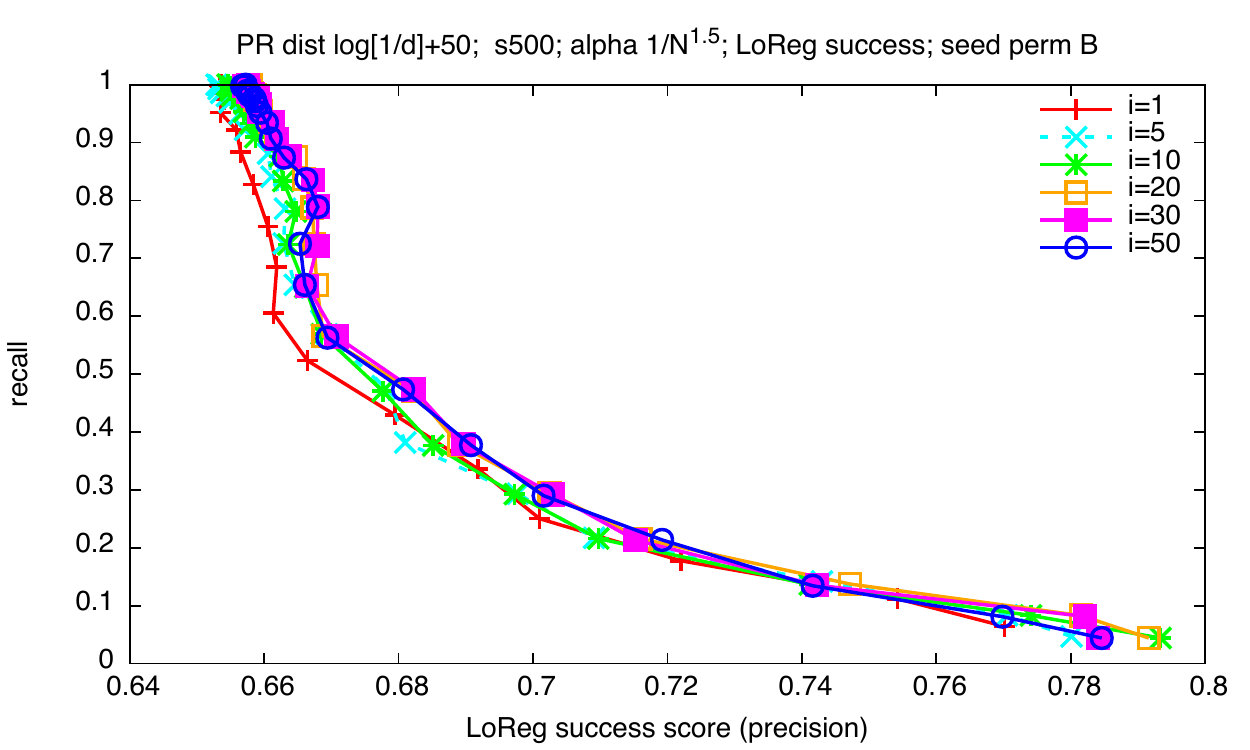}%
\\
\includegraphics[width=0.32\textwidth]{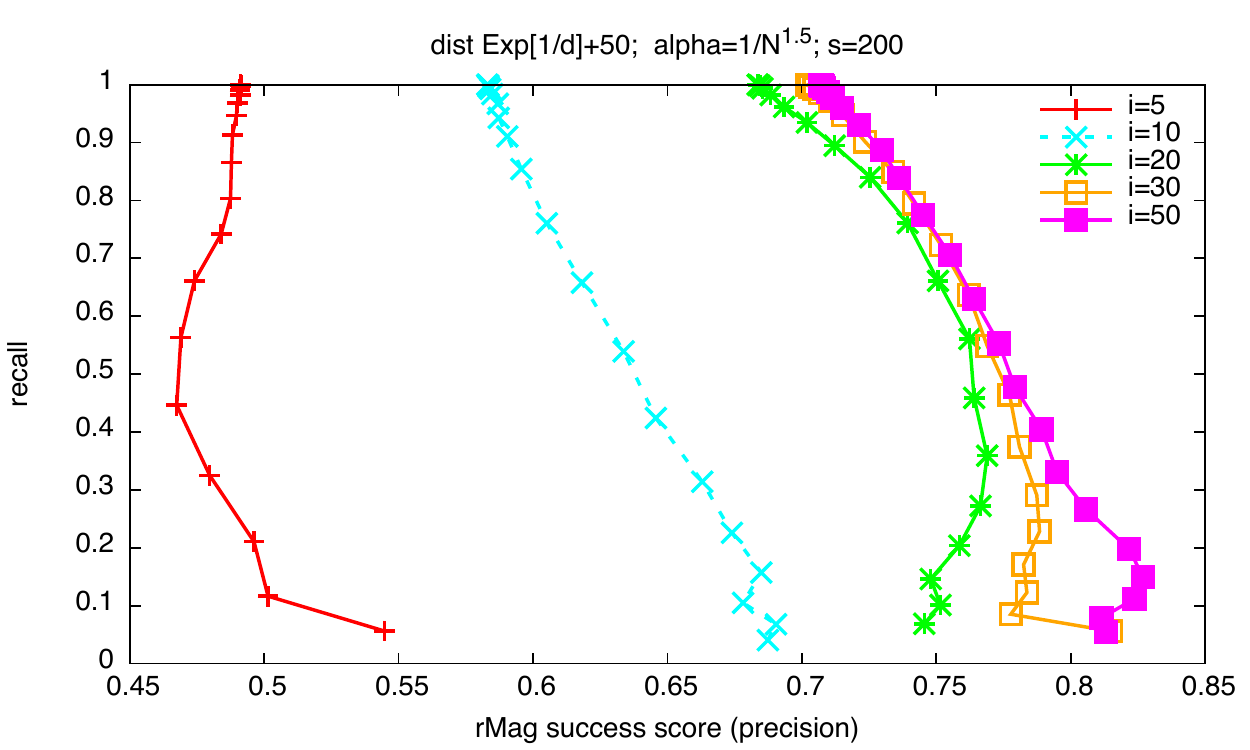}%
\includegraphics[width=0.32\textwidth]{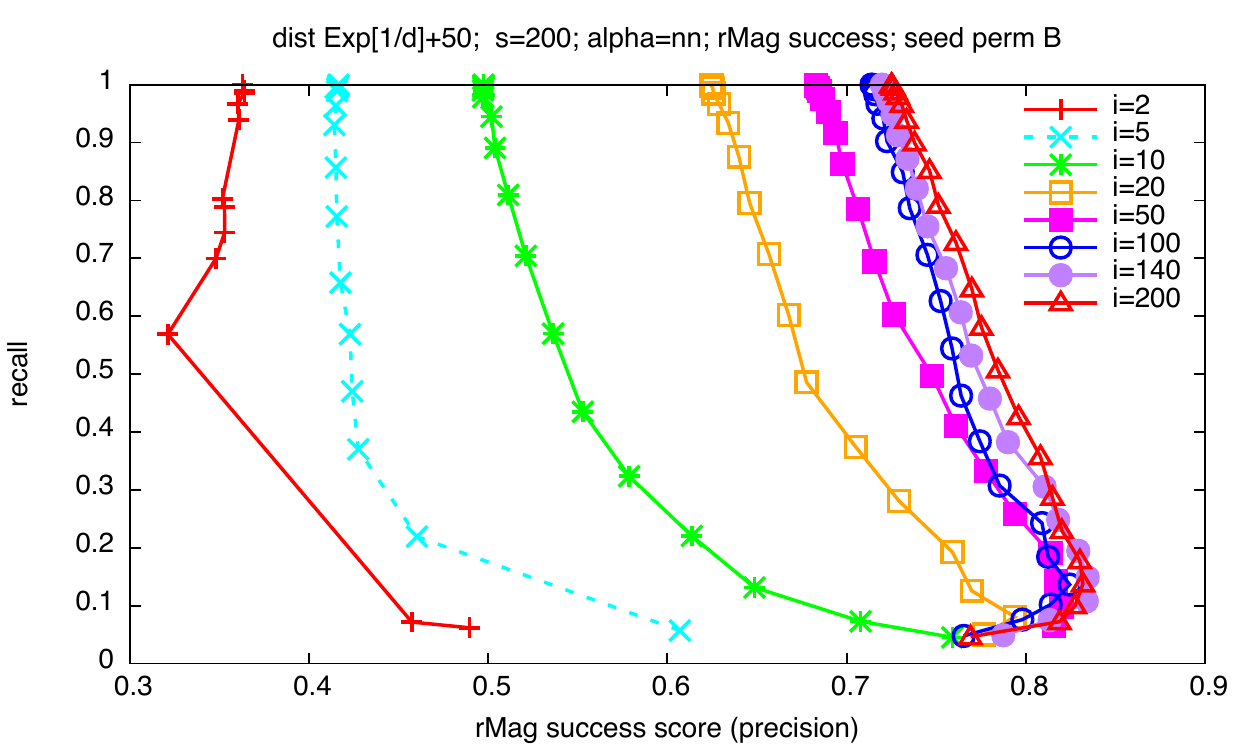}%
\includegraphics[width=0.32\textwidth]{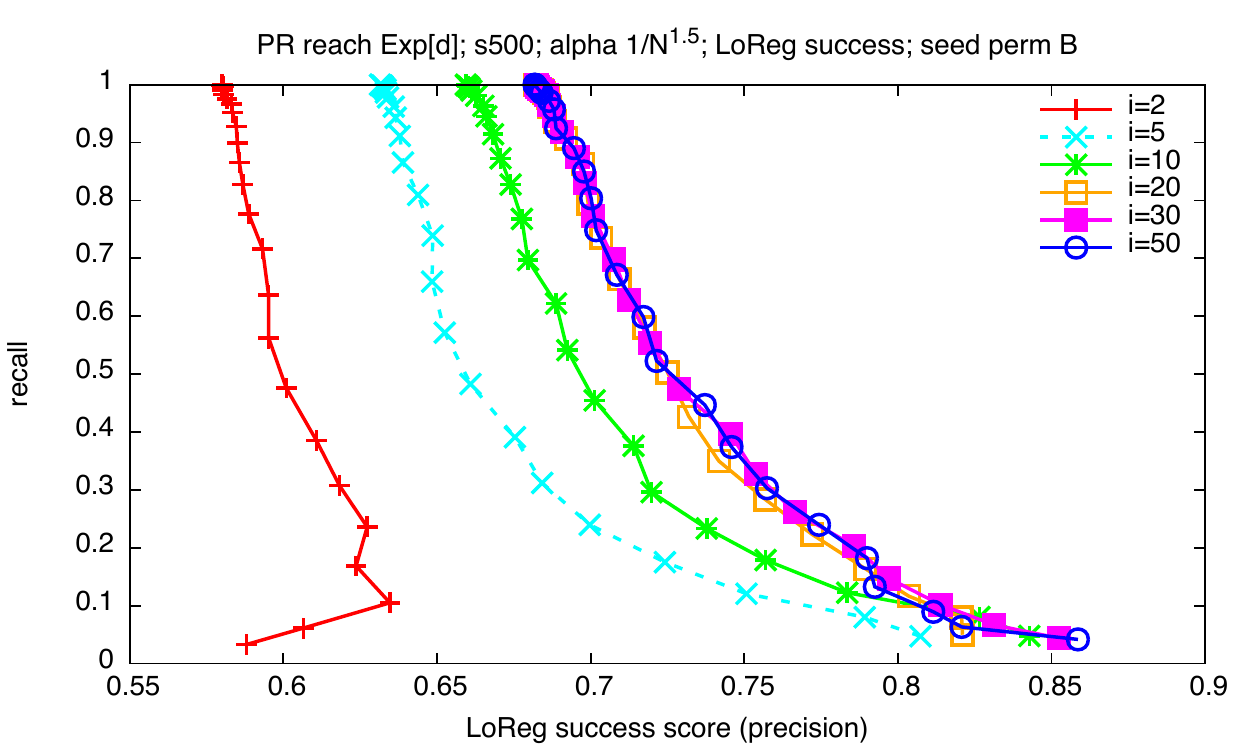}%

\caption{\label{sweepitermovielen:fig}{\instance{movielens1M}:
    \onlyinproc{Precision recall of different schemes as we sweep the
      number of simulations.}\notinproc{Precision recall of different
      schemes as we increase the number of simulations. Left and
      middle: Randomized
      distance diffusion Exp$[1/d]+50$ with rMag success scoring; left with
      $\alpha(x)=1/x^{1.5}$ and middle with $\alpha = nn$; top with
      $50$ seeds and bottom with $500$ seeds.   Top right:
      Fixed-length distance diffusion with $\log[1/d]+50$,
      $\alpha(x)=1/x^{1.5}$, and LoReg
      scoring, and 500 seeds.  Bottom right:  Reach diffusion Exp$[d]$
      with $\alpha(x)=1/x^{1.5}$, 500 seeds, and LoReg scoring. }}}  
\end{figure*}


\section{Conclusion} \label{conclu:sec}
We define {\em reach diffusion} and {\em distance diffusion} kernels
for graphs that are inspired by 
popular and successful measures of centrality and influence in social 
and economic networks.  We faciliate the application of our kernels  for
SSL by developing
highly scalable sketching algorithms.
We conducted a preliminary experimental evaluation demonstrating the
application and promise of our approach.
In future work, we hope to apply influence maximization algorithms for
active learning, that is, select the most effective seed sets for a
given labeling budget.  We also plan to explore the effectiveness of
our kernels as an alternative to spectral kernels 
with recent embedding techniques \cite{deepwalk:KDD2014,node2vec:kdd2016,Yang:ICML2016}.
\section*{Acknowledgment}
We would like to thank Fernando Pereira for discussions, pointers
to the literature, and sharing
views and intuitions on real-world challenges which
prompted the development of our proposed models.
{\scriptsize %
\bibliographystyle{plain}
\bibliography{cycle}
}

\appendix

\section{More Experiments}

\notinproc{
\subsection{Political blogs data}
The data consists of about 19,000 links between roughly 1200 blogs 
collected at the 2004 US presidential election. The blogs are labeled 
as {\em liberal} or {\em conservative},
 with half the blogs in each category. 
About 62\% of the blogs form a single 
strongly connected component, 21\%  can reach the component via 
links, and 16\% can only be reached from the component. 

 The label dimension here is $L=2$, as 
the provided label $\vecy_i$ of a blog $i$ is $(1,0)$ for {\em liberal} and $(0,1)$ for 
{\em conservative}.  
We used 10 sets of experiments. 
In each set, we select a different uniform random permutation of the 
blogs.  We then take seed sets $S$ of size $s \in [10,1000]$ as  prefixes of the same  permutation. 
Note that our seed set sizes range from less than $1\%$ to  about $83\%$  of all blogs. 
 We then apply our algorithms to compute learned labels $\vecf_j$ for all 
nodes $j$.  

To make a 
prediction, we consider the average label of the seed set 
$\overline{y}(S)$ (defined in \eqref{avelabel:eq}) and the learned label $\vecf_j$.  The prediction is then 
{\em liberal} if $f_{j1} > \overline{y}(S)_1$ and {\em conservative}
if $f_{j1} < \overline{y}(S)_1$.  If the  prediction is equal to the 
true label, we count it as a success. 
We define the {\em 
  margin} of our prediction as the 2-norm $||\overline{y}(S)-\vecy_j||_2$ of the difference 
between the average seed label and the learned label. 
When the margin is $0$, 
which happens when our model provides no information (no reachable seed nodes),
we take the success to be $0.5$. 
We consider the 
precision (fraction of successful predictions) and recall (number of 
predictions as fraction of total),  as a function of the margin. 

\begin{figure*}[t]
\centering 
\includegraphics[width=0.32\textwidth]{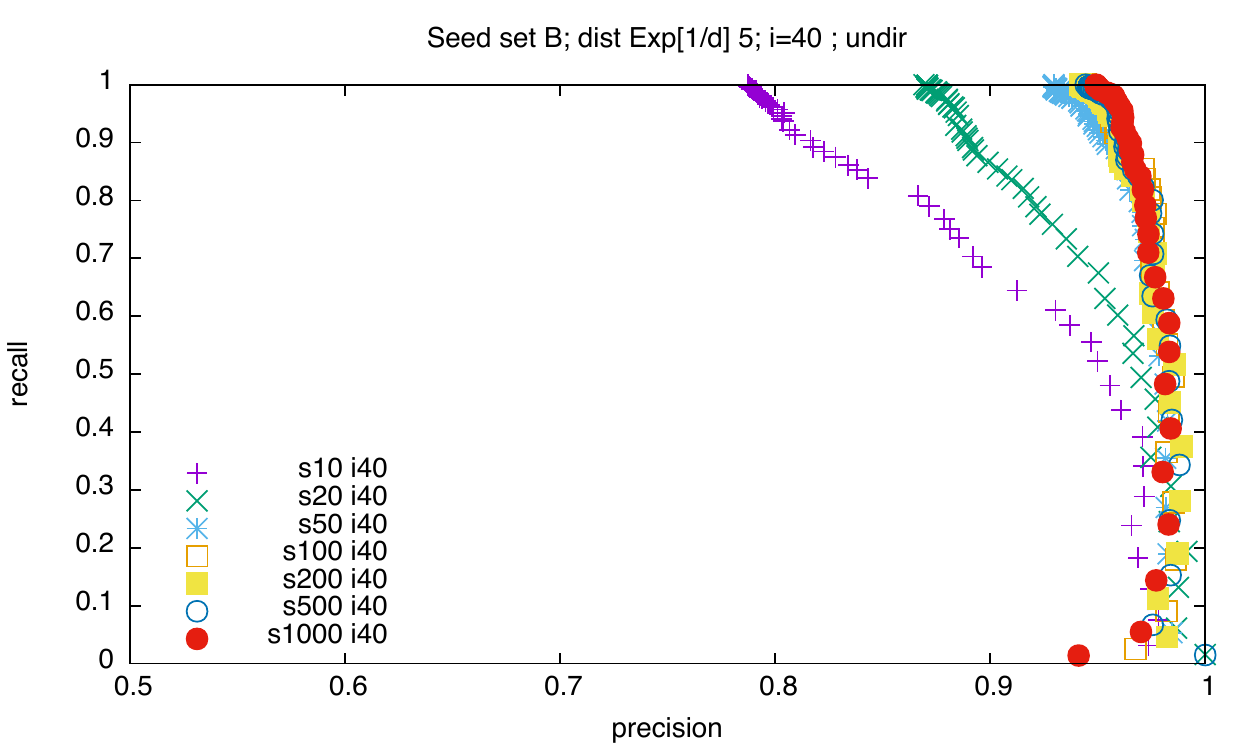}%
\includegraphics[width=0.32\textwidth]{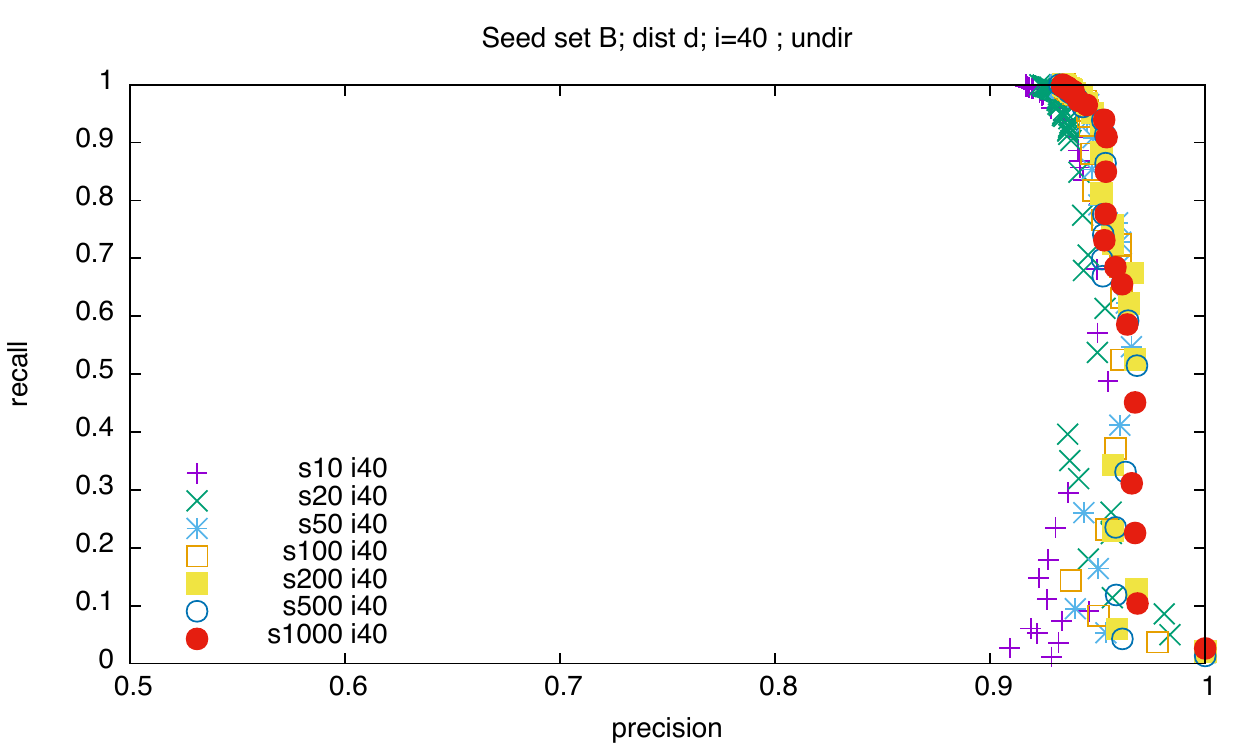}
\includegraphics[width=0.32\textwidth]{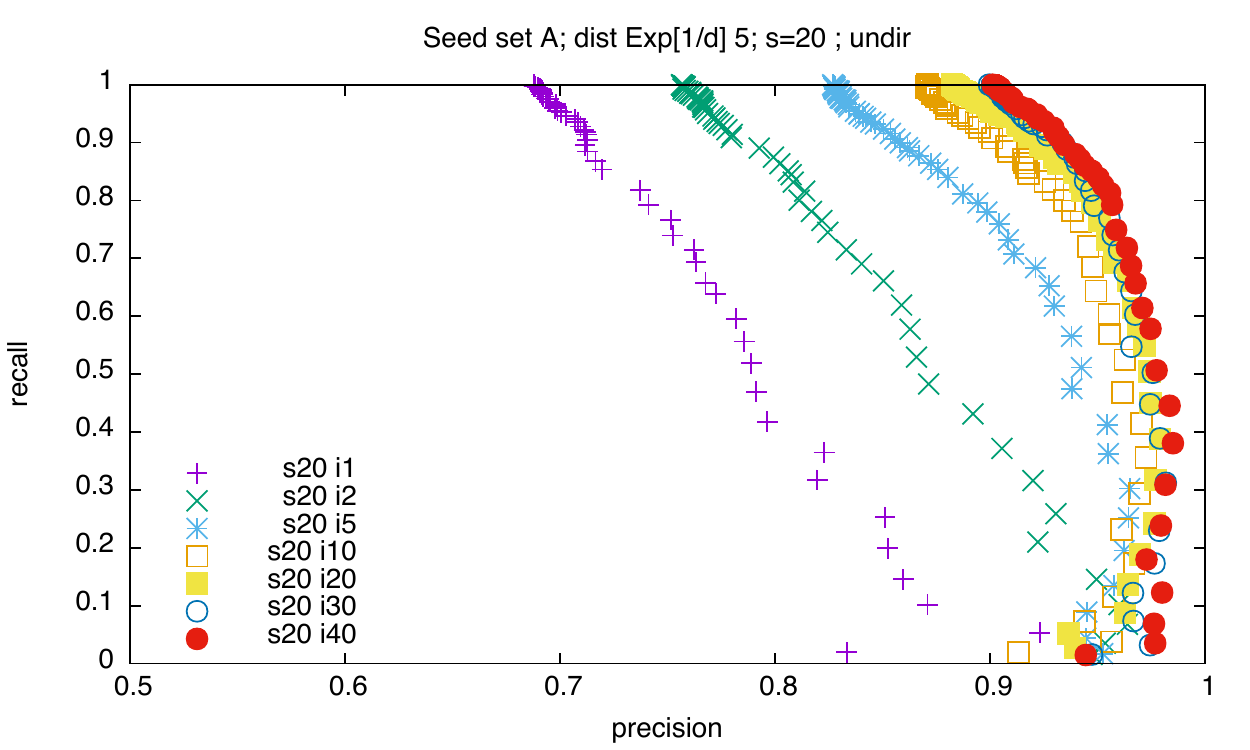}%
\caption{\label{polblogsweeps:fig}{\instance{polblog}: Precision 
    recall when sweeping the number of seeds: 
Left: randomized $\Exp[1/\Gamma(u)]+5$ with $40$ simulations.  Middle:
fixed lengths $\Gamma(u)$.  Right:  Sweeping the number of simulations 
with $s=50$ seeds and randomized $\Exp[1/\Gamma(u)]+5$.}}
\end{figure*}
Note that hyperlinks are directed, and the direction has a concrete 
semantics.  In our experiments, we separately worked with three sets of edges:
{\em forward} (an edges with the same direction is generated for each 
hyperlink),
{\em reversed} (a reversed edge is generated for each hyperlink), and 
{\em undirected} (two directed edges are generated for each hyperlink). 
 We then consider distance and reach diffusion on these directed graphs. 
As we did with the Movielens1M dataset, we used a limited selection of fixed-length and randomized length 
  (distance diffusion)   and lifetime (reach diffusion) schemes, as 
  outlined in Table~\ref{lengthpb:table}.  We used kernel weighting $\alpha(x)=1/x$. 
The function $|\Gamma(u)|$  is the outdegree of $u$, which is the 
number of hyperlinks to other blogs with {\em forward}, the number of 
hyperlinks to the blog with {\em reversed}, and the sum with {\em undirected}. 
  We used a sketch parameter $k=32$ and up to $40$ Monte Carlo simulations. 



\begin{table}
\caption{Lengths and lifetime schemes for Political blogs\label{lengthpb:table}}
{\scriptsize 
\begin{center}
\begin{tabular}{l|l|l}
Scheme name & specifications & parameters $g(x)$;$\delta$\\
\hline 
Dist Exp$[g(x)]+\delta$ &  $\ell_{uv} \sim \delta + \Exp[g(|\Gamma(u)|)]$ & $\frac{1}{x}$; $\delta\in \{0,1,5\}$ \\
& & $\frac{1}{\log_2(1+x)}$;  $\delta\in \{0,5\}$ \\
\hline 
Dist $g(x)$ (fixed) &  $\ell_{uv} = g(|\Gamma(u)|)$ xw & $\{1, \log_2(1+x),x\}$  \\
\hline 
Reach Exp$[g(x)]$ & $\mu_{uv}  \sim \Exp[g(|\Gamma(u)|)]$ & $\{ x, \log_2(1+x)\}$
\end{tabular}
\end{center}
}
\end{table}

  On this data set, randomization of lengths did not provide an 
  advantage.  
The fixed length schemes performed very well,   with $g(x)=x$ and 
  $g(x)=\log_2(1+x)$ being more consistent and slightly better than 
  $g(x)=1$.  For these schemes, there was no observable improvement with the number of simulations. 
The prediction success was typically over 90\% even with the smallest 
seed sets ($s=10$).  
This is explained by the two sets of blogs forming two distinct clusters,
detectable by most clustering algorithms.

  The best randomized distance scheme was $g(x)=1/x$ and $\delta=5$. 
  The best randomized reach scheme was $g(x)=x$.  Overall, the reach 
  diffusion schemes gave much weaker predictions than the distance 
  diffusion schemes and both were outperformed by the fixed-length schemes.

\begin{figure}[h]
\centering 
\includegraphics[width=0.23\textwidth]{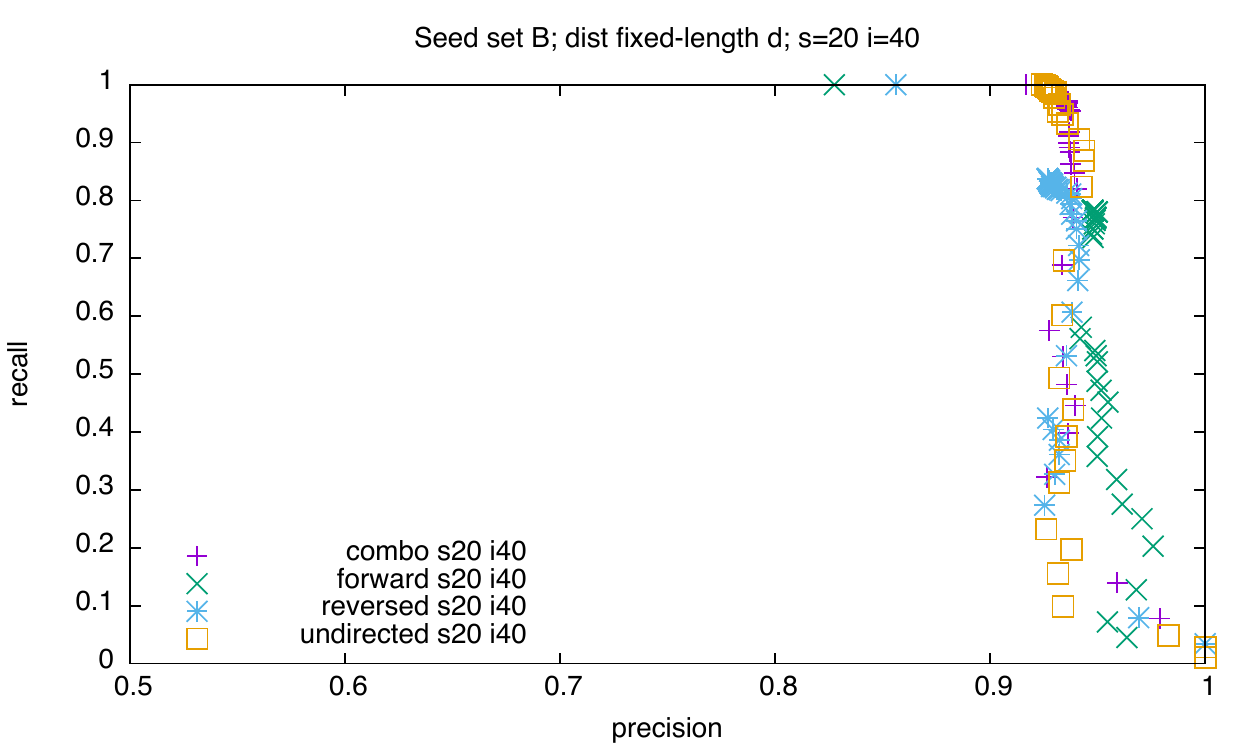}%
\includegraphics[width=0.23\textwidth]{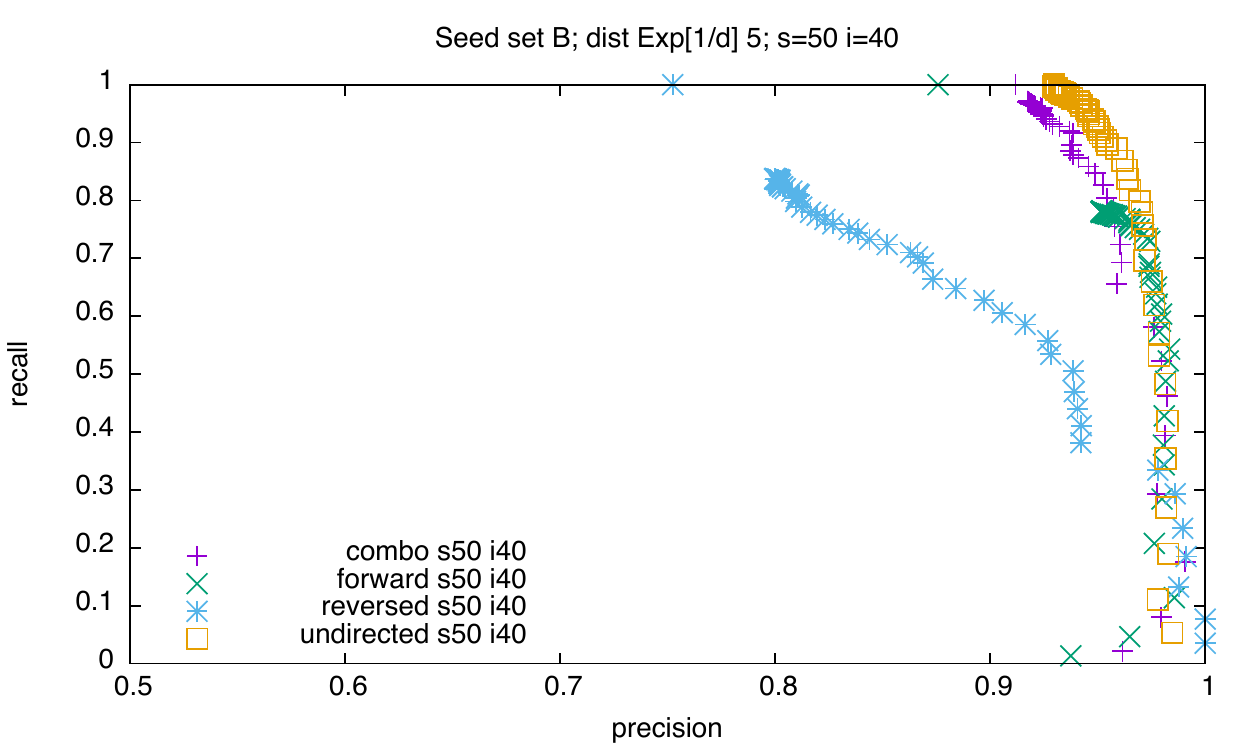}%
\caption{\label{polblogsdir:fig}{\instance{polblog}: Precision recall for different 
    directions, distance diffusion with fixed-lengths of $\Gamma(u)$
    and $s=20$
    (left) and randomized $\Exp[1/\Gamma(u)]+5$ with $s=50$ (right)}}
\end{figure}
  The randomized distance and reach diffusion schemes did  show drastic improvement with the number of 
  simulations (see Figure~\ref{polblogsweeps:fig} (right)). 
   All schemes were more accurate with larger seed sets (see 
   Figure~\ref{polblogsweeps:fig} (left and middle)). 
Performance did strongly depend on direction (see 
Figure~\ref{polblogsdir:fig} for representative results):   {\em Reversed}
was clearly inferior to {\em forward}. 
  {\em Undirected} and {\em forward} were comparable and consistently best, with 
   the former providing a higher recall.  We also evaluated combined 
   predictions ({\em combo}), which go with the prediction with the largest margin 
   among {\em forward} and {\em reversed}.  Prediction quality of {\em 
     combo} was more consistent 
   than {\em reversed} but was dominated by {\em undirected} and {\em forward}.

} 

\end{document}